\documentclass[11pt]{article}
\usepackage[margin=1.1in, a4paper]{geometry}

\usepackage[utf8]{inputenc}

\usepackage[utf8]{inputenc}

\usepackage{graphicx}
\usepackage{subfig}

\usepackage{amsthm}
\usepackage{amsmath}
\usepackage{amsfonts}
\usepackage{gensymb}
\usepackage{dsfont}
\usepackage{amssymb}
\usepackage{centernot}
\usepackage{mathtools}
\usepackage{stmaryrd }
\usepackage{bbm}
\usepackage{bm}

\usepackage{hyperref}
\usepackage{xurl}

\usepackage{textcomp}
\usepackage{breqn}
\usepackage{enumerate}
\usepackage[dvipsnames]{xcolor}

\usepackage{nomencl}

\usepackage[toc,page]{appendix}

\DeclareMathOperator*{\argmin}{argmin}

\newtheorem{theorem}{Theorem}[section]

\newtheorem{lemma}[theorem]{Lemma}
\newtheorem{proposition}[theorem]{Proposition}
\theoremstyle{definition}
\newtheorem{definition}{Definition}[section]
\newtheorem*{remark}{Remark}
\newtheorem{assumption}{Assumption}

\newcommand\ie{\textit{i.e.}}
\newcommand\resp{\textit{resp.}}
\newcommand\wrt{\textit{w.r.t.}}

\newcommand\eg{\textit{e.g.}}

\newcommand{\quoting}[1]{``#1''}
\newcommand{\iid}{\text{i.i.d.}}
\newcommand{\aka}{\textit{\text{a.k.a.}}}

\newcommand{\innerprod}[2]{\left<#1, \, #2 \right>}

\newcommand{\straightd}{\mathrm{d}}

\usepackage{authblk}

\usepackage{graphicx}
\usepackage{subfig}
\usepackage{wrapfig}

\usepackage{booktabs}
\usepackage{diagbox}

\usepackage{amsfonts}
\usepackage{gensymb}
\usepackage{dsfont}
\usepackage{amssymb}
\usepackage{centernot}
\usepackage{mathtools}
\usepackage{stmaryrd }
\usepackage{relsize}

\usepackage{hyperref}
\usepackage{xurl}

\usepackage{textcomp}
\usepackage{breqn}
\usepackage{enumerate}
\usepackage[dvipsnames]{xcolor}

\usepackage{nomencl}
\usepackage[round]{natbib}

\usepackage[toc,page]{appendix}

\newcommand{\tildeGamma}{\Tilde{\gamma}}
\newcommand{\tildeT}{\Tilde{T}}
\newcommand{\wstar}{w^\star}

\newcommand{\lambdamin}{\lambda_{\text{\tiny min}}}

\hypersetup{
    colorlinks=true,
    citecolor=blue,
    linkcolor=blue,
    filecolor=magenta,      
    urlcolor=magenta,
}

\numberwithin{equation}{section}

\begin{document}

% \title{Learning structure in the presence of symmetries with infinitely wide neural networks}

\title{On the symmetries in the dynamics of wide two-layer neural networks}
% \title{Symmetries in the dynamics of wide two-layer neural networks}

\author[,1]{Karl Hajjar\thanks{Corresponding author: \texttt{hajjarkarl@gmail.com}}}
\author[2]{Lénaïc Chizat}
% \author[1]{Christophe Giraud}
\affil[1]{Laboratoire de Mathématiques d'Orsay\\ Université Paris-Saclay\\ 91405 Orsay, France}
\affil[2]{Institut de Mathématiques\\
       École Polytechnique Fédérale de Lausanne\\
       Lausanne, Switzerland}

\maketitle

% \begin{abstract}%
% \looseness=-1
%     The success of neural networks is believed to be in part due to the learning of representations adapted to the structure of the problem and allowing to predict the correct output. We study the relationship between the symmetries of the problem and the structure of the learned predictor for infinitely wide two-layer ReLU networks trained with gradient flow. We show that the orthogonal symmetries of the target function $f^*$ are shared by the learned predictor under gradient flow. In particular, when $f^*$ is odd, these symmetries lead to a linear predictor for which exponential convergence to the global minimum can be obtained. When $f^*$ has a hidden low-dimensional structure, we prove that the gradient flow PDE reduces to a lower-dimensional PDE. Furthermore, we present informal and numerical arguments pointing towards the adaption of predictor to the lower-dimensional structure of the problem.
% \end{abstract}

\begin{abstract}
    
    We consider the idealized setting of gradient flow on the population risk for infinitely wide two-layer ReLU neural networks (without bias), and study the effect of symmetries on the learned parameters and predictors. We first describe a general class of symmetries which, when satisfied by the target function $f^*$ and the input distribution, are preserved by the dynamics. We then study more specific cases. When $f^*$ is odd, we show that the dynamics of the predictor reduces to that of a (non-linearly parameterized) linear predictor, and its exponential convergence can be guaranteed. When $f^*$ has a low-dimensional structure, we prove that the gradient flow PDE reduces to a lower-dimensional PDE. Furthermore, we present informal and numerical arguments that suggest that the input neurons align with the lower-dimensional structure of the problem.
\end{abstract}

\section{Introduction}

The ability of neural networks to learn rich representations---or features---of their input data is commonly observed in state-of-the art models~\cite{zeiler2014visualizing, cammarata2020thread} and often thought to be the reason behind their good practical performance~\cite[Chap.~1]{Goodfellow-et-al-2016}. Yet, our theoretical understanding of how feature learning arises from simple gradient-based training algorithms remains limited. Much progress (discussed in Section~\ref{sec:related-work}) has been made recently to understand the power and limitations of gradient-based learning with neural networks, showing in particular their superiority over fixed-feature methods on some difficult tasks. However, positive results are often obtained for algorithms that differ in substantial ways from plain (stochastic) gradient descent (e.g.~the layers trained separately, or the algorithm makes just one truly non-linear step, etc). 

In this work, we take the algorithm as a given and instead adopt a descriptive approach. Our goal is to improve our understanding of how neural networks behave in the presence of symmetries in the data with plain gradient descent (GD) on two-layer fully-connected ReLU neural networks. To this end, we investigate situations with strong symmetries on the data, the target function and on the initial parameters, and study the properties of the training dynamics and the learned predictor in this context.

\subsection{Problem setting}\label{sec:pb-setting}
We denote by $d$ the input dimension, $\rho$ the input data distribution which we assume to be uniform over the unit sphere $\mathbb{S}^{d-1}$ of $\mathbb{R}^d$, and by $\mathcal{P}_2(\Omega)$ the space of probability measures with finite second moments over a measurable space $\Omega$.
We call $\sigma$ the activation function, which we take to be \text{ReLU}, that is $\sigma(z) = \max(0, z)$, 
$\ell : \mathbb{R} \times \mathbb{R} \to \mathbb{R}$ the loss function, which we assume to be continuous in both arguments and continuously differentiable \wrt~its second argument and we denote by $\partial_2 \ell$ this derivative. 

\paragraph{Mean-field limit of two-layer networks.}

In this work, we consider the infinite-width limit in the mean-field regime of the training dynamics of two-layer networks without intercept with a \text{ReLU} activation function. Given a measure $\mu \in \mathcal{P}_2(\mathbb{R} \times \mathbb{R}^d)$, we consider the infinitely wide two-layer network parameterized by $\mu$, defined, for any input $x \in \mathbb{R}^d$, by
\begin{align}\label{eq:predictor}
    f(\mu; x) = \int_{c \in \mathbb{R}^{1+d}} \phi(c; x) \straightd\mu(c),
\end{align}
where, for any $c=(a,b) \in \mathbb{R} \times \mathbb{R}^d$, $\phi(c;x) = a \sigma \left(b^\top x \right)$. Note that width-$m$ two-layer networks with input weights $(b_j)_{j \in [1, m]} \in (\mathbb{R}^d)^{m}$ and output weights $(a_j)_{j \in [1, m]} \in \mathbb{R}^m$ can be recovered by a measure $\mu_m = (1/m) \sum_{j=1}^m \delta_{(m a_j, b_j)}$ with $m$ atoms. 

\paragraph{Objective and Wasserstein gradient flow.}
We consider the problem of minimizing the \textit{population loss} objective for a given target function $f^*:\mathbb{R}^d \rightarrow \mathbb{R}$, which we assume to be bounded on the unit sphere, that is
\begin{align}\label{eq:objective}
    \min_{\mu \in \mathcal{P}_2(\mathbb{R} \times \mathbb{R}^d)} \Big(F(\mu) := \mathbb{E}_{x \sim \rho} \left[ \ell \left(f^*(x), f(\mu; x) \right) \right] \Big).
\end{align}
The Fréchet derivative of the objective function $F$ at $\mu$ is given by the function $F^{\prime}_{\mu}(c) = \mathbb{E}_{x \sim \rho} \left[\partial_2 \ell \left(f^*(x), f(\mu; x) \right) \phi(c;x) \right]$ for any $c=(a,b) \in \mathbb{R} \times \mathbb{R}^d$ (for more details, see Appendix~\ref{app:first-variation}). Starting from a given measure $\mu_0 \in \mathcal{P}_2(\mathbb{R} \times \mathbb{R}^d)$, we study the Wasserstein gradient flow (GF) of the objective~\eqref{eq:objective} which is a path $(\mu_t)_{t \geq 0}$ in the space of probability measures satisfying, in the sense of distributions, the partial differential equation (PDE) known as the continuity equation:
\begin{equation}\label{eq:w-gf}
    \begin{aligned}
        \partial_t \mu_t &= -\text{div} \left(v_t \, \mu_t \right), \\
        v_t(c) :&= -\nabla F^{\prime}_{\mu_t}(c).
    \end{aligned}
\end{equation}
\paragraph{Initialization.} 

We make the following assumption on the initial measure $\mu_0 \in \mathcal{P}_2(\mathbb{R} \times \mathbb{R}^d)$: $\mu_0$ decomposes as $\mu_0 = \mu_0^1 \otimes \mu_0^2$ where $\mu_0^1, \mu_0^2 \in \mathcal{P}_2(\mathbb{R}) \times \mathcal{P}_2(\mathbb{R}^d)$. This follows the standard initialization procedure at finite width. Because no direction should \textit{a priori} be favored, we assume $\mu_0^2$ to have spherical symmetry, \ie, it is invariant under any orthogonal transformation, and we additionally assume that $|a| = ||b||$ almost surely at initialization. It is shown in~\cite[Lemma 26]{chizat2020implicitBias}, and \cite[Section 2.5]{wojtowytsch2020convergence}, that with this assumption, $\mu_t$ stays supported on the set $\{|a| = ||b||\}$ for any $t \geq 0$. 

\paragraph{Comment on the assumptions.}
The assumption that $\mu_0$ decomposes as a product of two measures is to stay as close as possible to what is done in practice (independent initialization for different layers). The assumption that 
$|a| = ||b||$ is of a technical nature, and, along with the regularity conditions on the loss $\ell$ and the input data distribution $\rho$, ensures that the Wasserstein GF~\eqref{eq:w-gf} is well-defined~\cite[Lemma 3.1, Lemma 3.9]{wojtowytsch2020convergence} when using ReLU as a activation function (which bears technical difficulties because of its non-smoothness). The results of Section~\ref{sec:sym} hold for others activation function which potentially require less restrictive assumptions on $\mu_0$ and $\rho$ but still requires $\mu_0$ to decompose as a product of measures. In contrast, the results of Sections~\ref{sec:odd-case} and~\ref{sec:low-dim} are specific to $\sigma=\text{ReLU}$ and thus require the assumptions above on $\mu_0$ and $\rho$. Since our work focuses mostly on ReLU, we choose to state the results of all sections with the (more restrictive) assumptions stated above on $\mu_0$ and $\rho$.

\paragraph{Relationship with finite-width GD.}
If $\mu_0 = (1/m) \sum_{j=1}^m \delta_{(a_j(0), b_j(0))}$ is discrete, the Wasserstein GF~\eqref{eq:w-gf} is exactly continuous-time GD on the parameters of a standard finite-width neural network, and discretization errors (\wrt~the number of neurons) can be provided~\cite{mei2018mean, nguyen2020rigorous}.

\subsection{Summary of contributions}

Our main object of study is the gradient flow of the \emph{population} risk of \emph{infinitely wide} two-layer ReLU neural networks without intercept. Our motivation to consider this idealistic setting---infinite data and infinite width---is that it allows, under suitable choices for $\rho$ and $\mu_0$, the emergence of exact symmetries which are only approximate in the non-asymptotic setting\footnote{In contrast, our focus on GF is only for theoretical convenience and most of our results could be adapted to the case of GD.}.
%We also assume that the initial distribution of the parameters is spherically symmetric and that the response variable is a function $f^*$ of the input.

\paragraph{Symmetries, structure, and convergence.}
In this work, we are interested in the structures learned by the predictor $f(\mu_t; \cdot)$ under GF as $t$ grows large. Specifically, we make the following contributions:
\begin{itemize}
    \item In Section~\ref{sec:sym}, we prove that if $f^*$ is invariant under some orthogonal linear map $T$, then $f(\mu_t; \cdot)$ inherits this invariance under GF (\autoref{th:learning-inv}).
    % we assume $f^*$ is \textit{invariant} under some orthogonal transformation $T$ and study when the learned predictor $f(\mu_t; \cdot)$ enjoys the same invariance: .
    % under which assumptions on $\mu_0$ and $\rho$ the predictor $f(\mu_t; \cdot)$ enjoys the same invariance. We prove that if the input data distribution $\rho$ is invariant under some orthogonal transformation of the input (Definition~\ref{def:func-inv}), then the learnt predictor inherits this invariance (\autoref{th:learning-inv})
    
    \item In Section~\ref{sec:odd-case}, we study the case when $f^*$ is an \textit{odd} function and show that the network converges to the best linear approximator of $f^*$ at an exponential rate (\autoref{th:lin-exp-cv}). Linear predictors are optimal over the hypothesis class in that case, in particular because there is no intercept in our model.
    % the dynamics are degenerate and converge to the best linear approximator of $f^*$ at an exponential rate.
    
    \item
    In Section~\ref{sec:low-dim}, we consider the \textit{multi-index model} where $f^*$ depends \textit{only} on the orthogonal projection of its input onto some sub-space $H$ of dimension $d_H$. We prove that the dynamics can be reduced to a PDE in dimension $d_H$. If in addition, $f^*$ is the Euclidean norm of the projection of the input, we show that the dynamics reduce to a one-dimensional PDE (\autoref{th:1d-theta}). In the latter case, we were not able to prove theoretically the convergence of the neurons of the first layer towards $H$, and leave this as an open problem but we provide numerical evidence in favor of this result.
\end{itemize}
The code to reproduce the results of the numerical experiments can be found at: \\ \url{https://github.com/karl-hajjar/learning-structure}.

\subsection{Related work}\label{sec:related-work}

\paragraph{Infinite-width dynamics.}
It has been shown rigourously that for infinitely wide networks there is a clear distinction between a feature-learning regime and a kernel regime~\cite{chizat2019lazyTraining, yang2020featureLearning}. For shallow networks, this difference stems from a different scale (\wrt~width) of the initialization where a large initialization leads to the Neural Tangent Kernel (NTK) (\aka~the \quoting{lazy regime}) which is equivalent to a kernel method with random features~\cite{jacot2018ntk} whereas a small initialization leads to the so-called
\textit{mean-field} (MF) limit where features are learned from the first layer~\cite{chizat2019lazyTraining, yang2020featureLearning}. However, it is unclear in this setting exactly what those features are and what underlying structures are learned by the network. The aim of the present work is to study this phenomenon from a theoretical perspective for infinitely wide networks and to understand the relationship between the ability of networks to learn specific structures and the symmetries of a given task. 

A flurry of works study the dynamics of infinitely wide two-layer neural networks.~\cite{chizat2018global, mei2018mean, rotskoff2018parameters, wojtowytsch2020convergence, sirignano2020mean} study the gradient flow dynamics of the MF limit and show that they are well-defined in general settings and lead to convergence results (local or global depending on the assumptions). On the other hand, Jacot et.~al~\cite{jacot2018ntk} study the dynamic of the NTK parameterization in the infinite-width limit and show that it amounts to learning a \textit{linear} predictor on top of random features (fixed kernel), so that there is no feature learning.

\paragraph{Convergence rates.} 
% \citet{jacot2018ntk} describe the evolution of the learned predictor in the NTK parameterization as a \quoting{kernel descent} and the convergence to the target function at an exponential rate is readily obtained. 

In the MF limit, convergence rates are in general difficult to obtain in a standard setting. For instance,
\cite{chizat2018global, wojtowytsch2020convergence} show the convergence of the GF to a global optimum in a general setting but this does not allow convergence rates to be provided. 
To illustrate the convergence of the parameterizing measure to a global optimum in the MF limit, E et.~al~\cite{e2020continuousML} prove \textit{local} convergence (see Section 7) for one-dimensional inputs and a specific choice of target function in $O(t^{-1})$ where $t$ is the time step. At finite-width, Daneshmand and Bach~\cite{daneshmand2022polynomial} also prove convergence of the parameters to a global optimum in $O(t^{-1})$ using an algebraic idea which is specific to the ad-hoc structure they consider (inputs in two dimensions and target functions with finite number of atoms).
%but their analysis is limited to inputs in two dimensions, specific target functions which only have a finite number of values, and predictors which have a finite number of neurons which is also that of the target function.

In Section~\ref{sec:odd-case}, we show convergence of the MF limit at an exponential rate when the target function is odd. In the setting of this section, the training dynamics are degenerate and although input neurons move, the symmetries of the problem imply that the predictor is linear.

%
% In a different setting, \citet{daneshmand2022polynomial} show convergence of the parameters to a global optimum in $O(t^{-1})$ where $t$ is the time step, but their analysis is limited to inputs in two dimensions, specific target functions which only have a finite number of values, and predictors which have a finite number of neurons which is also that of the target function. Similarly, for one-dimensional inputs and a specific choice of target function,~\citet{e2020continuousML} also show \textit{local} convergence (see Section 7) of the parameterizing measure to a global optimum in $O(t^{-1})$.~\citet{chen2022feature} show that the convergence occurs at an exponential rate for two-layer networks for much more general target functions and input dimensions but only when the number of samples does not exceed the input dimension and the output layer weights are fixed, which does not cover our setting. 

% \citet{chen2022feature} are able to prove convergence at a linear rate for two-layer networks for much more general target functions and input dimensions, but only if the number of samples does not exceed the dimension of the inputs. In contrast, our results are valid without restriction on the dimension but for a class of target functions which bear specific symmetries.  

\paragraph{Low-dimensional structure.} 

Studying how neural networks can adapt to hidden low-dimensional structures is a way of approaching theoretically the feature-learning abilities of neural networks. Bach~\cite{bach2017breaking} studies the statistical properties of infinitely wide two-layer networks, and shows that when the target function only depends on the projection on a low-dimensional sub-space, these networks circumvent the curse of dimensionality with generalization bounds which only depend on the dimension of the sub-space. In a slightly different context, Chizat and Bach~\cite{chizat2020implicitBias} show that for a binary classification task, when there is a low-dimensional sub-space for which the projection of the data has sufficiently large inter-class distance, only the dimension of the sub-space (and not that of the ambient space) appears in the upper bound on the probability of misclassification. Whether or not such a low-dimensional sub-space is actually learned by GD is not addressed in these works.

Similarly,~\cite{cloninger2021deep, damian2022neural} focus on learning functions which have a hidden low-dimensional structure with neural networks. They consider a single step of GD on the input layer weights and show that the approximation / generalization error adapts to the structure of the problem: they provide bounds on the number of data points / parameters needed to achieve negligible error, which depend on the reduced dimension and not the dimension of the ambient space. In a similar context, Mousavi-Hosseini et.~al~\cite{mousavi2022neural} consider (S)GD on the first layer only of a finite-width two-layer network and show that with sufficient $L_2$-regularization and with a standard normal distribution on the input data the first layer weights align with the lower-dimensional sub-space when trained for long enough. They then use this property to then provide statistical results on networks trained with SGD. 

In a setting close to ours but on a classification task with finite-data and at finite-width, Paccolat et.~al~\cite{Paccolat_2021} compare the feature learning regime with the NTK regime in the presence of hidden low-dimensional structure and quantify for each regime the scaling law of the test error \wrt~the number of training samples, mostly focusing on the case $d_H=1$.

In a similar setting to that of~\cite{bach2017breaking}, Abbe et.~al~\cite{abbe2022staircase} study how GF for infinitely wide two-layer networks can learn specific classes of functions which have a hidden low-dimensional structure when the inputs are Rademacher variables. This strong symmetry assumption ensures that the learned predictor shares the same low-dimensional structure at any time step (from the $t=0$) and this allows them to characterize precisely what classes of target functions can or cannot be learned by GF in this setting. In contrast, we are interested in how infinitely wide networks \textit{learn} those low-dimensional structures during training, and in
the role of symmetries in enabling such a behaviour after initialization.

\paragraph{Learning representations.}
An existing line of work~\cite{yehudai2019power, allen2019learning, abbe2021staircase, damian2022neural, ba2022high} studies in depth the representations learned by neural networks trained with (S)GD at finite-width from a different perspective focusing on the advantages of feature-learning in terms of performance comparatively to using random features. In contrast, our aim is to describe the representations themselves in relationship with the symmetries of the problem.

\paragraph{Symmetries.} 
We stress that the line of work around symmetries of neural networks dealing with finding network architectures for which the output is invariant (\wrt~to its input or parameters) by some \textit{group} of transformations (see~\cite{bloem2020probabilisticSymmetries, ganev2021qr, gluch2021noether}, and references therein) is entirely different from what we are concerned with in the present work. In contrast, the setting of Mei et.~al~\cite{mei2018mean} is much closer to ours as they study how the invariances of the target function / input data can lead to simplifications in the dynamics of infinitely wide two-layer networks in the mean-field regime which allows them to prove global convergence results.

\subsection{Notations}

We denote by $\mathcal{M}_+(\Omega)$ the space of non-negative measures over a measurable space $\Omega$. For any measure $\mu$ and measurable map $T$, $T_{\#} \mu$ denotes the pushforward measure of $\mu$ by $T$. We denote by $\mathcal{O}(p)$ and $\text{id}_{\mathbb{R}^p}$ respectively the orthogonal group and the identity map of $\mathbb{R}^p$ for any $p \in \mathbb{N}$. Finally, $\innerprod{\cdot}{\cdot}$ is the Euclidean inner product and $||\, \cdot \,||$ the corresponding norm.

\section{Invariance under orthogonal symmetries}\label{sec:sym}

In this section, we demonstrate that if the target function $f^*$ is invariant under some orthogonal transformation $T$, since the input data distribution is also invariant under $T$, then $f(\mu_t; \cdot)$ is  invariant under $T$ as well for any $t \geq 0$. This invariance property of the dynamics \wrt~orthogonal symmetries is possible with an infinite number of neurons but is only approximate at finite-width. It is noteworthy that the results of this section hold for any activation function $\sigma$ and input data distribution $\rho$ which has the same symmetries as $f^*$,  provided that the Wasserstein GF~\eqref{eq:w-gf} is unique. We start with a couple of definitions:

\begin{definition}[Function invariance]\label{def:func-inv}
Let $T$ be a map from $\mathbb{R}^d$ to $\mathbb{R}^d$, and $f: \mathbb{R}^d \rightarrow \mathbb{R}$. Then, $f$ is said to be \textit{invariant} (\resp~\textit{anti-invariant}) under $T$ if for any $x \in \mathbb{R}^d$, $f(T(x)) = f(x)$ (\resp~$f(T(x)) = -f(x)$).
\end{definition}

\begin{definition}[Measure invariance]\label{def:meas-inv}
Let $\Omega \subset \mathbb{R}^d$, $T$ be a measurable map from $\Omega$ to $\Omega$, and $\mu$ be a measure on $\Omega$. Then, $\mu$ is said to be invariant under $T$ if $T_{\#} \mu = \mu$, or equivalently, if for any continuous and compactly supported $\varphi: \Omega \to \mathbb{R}$, $\int \varphi(x) \straightd \mu(x) = \int \varphi(T(x)) \straightd \mu(x)$.
\end{definition}
We are now ready to state the two main results of this section.

\begin{proposition}[Learning invariance]\label{th:learning-inv}
Let $T \in \mathcal{O}(d)$, and assume that $f^*$ is invariant under $T$. Then, for any $t \geq 0$, the Wasserstein GF $\mu_t$ of Equation~\eqref{eq:w-gf} is invariant under $\Tilde{T}: (a, b) \in \mathbb{R} \times \mathbb{R}^d \mapsto (a, T(b))$, and the corresponding predictor $f(\mu_t; \cdot)$ is invariant under $T$. 
\end{proposition}

\begin{proposition}[Learning anti-invariance]\label{th:learning-anti-inv}
Under the same assumptions as in Proposition~\ref{th:learning-inv} except now we assume $f^*$ is anti-invariant under $T$, and assuming further that $\partial_2 \ell(-y, -\hat{y}) = -\partial_2 \ell(y, \hat{y})$ for any $y, \hat{y} \in \mathbb{R}$, and that $\mu_0^1$ is symmetric around $0$ (\ie, invariant under $:a \in \mathbb{R} \mapsto - a$), we then have that  for any $t \geq 0$, the Wasserstein GF $\mu_t$ in Equation~\eqref{eq:w-gf} is invariant under $\Tilde{T}: (a, b) \in \mathbb{R} \times \mathbb{R}^d \mapsto (-a, T(b))$, and the corresponding predictor $f(\mu_t; \cdot)$ is anti-invariant under $T$.
\end{proposition}

\begin{remark}
The results above also hold for networks with intercepts at both layers. The conditions of Proposition~\ref{th:learning-anti-inv} are satisfied by both the squared loss and the logistic loss (\aka~the cross-entropy loss).
\end{remark}

Essentially, those results show that training with GF preserves the orthogonal symmetries of the problem:
the invariance of the target function under an orthogonal transformation  leads to the same invariance for $\mu_t$ and $f(\mu_t; \cdot)$. 
The proof, presented in Appendix~\ref{app:sym}, relies crucially on the fact that $T$ is an orthogonal map which combines well with the structure of $\phi(c;x)$ involving an inner product. The idea is essentially that the orthogonality of $T$ allows us to relate the gradient of $\phi$ (and consequently of $F^\prime_{\mu_t}$) \wrt~$c$ at $(T(c); x)$ to the same gradient at $(c;T^{-1}(x))$ and then to use the invariance of $f^*$ and $\rho$ to conclude. 

In the following sections we discuss the particular cases where functions are (anti-)invariant under $-\text{id}_{\mathbb{R}^d}$ (\ie, even or odd functions) or some sub-group of $\mathcal{O}(d)$.

\section{Exponential convergence for odd target functions}\label{sec:odd-case}
We consider here an odd target, function, \ie, for any $x \in \mathbb{R}^d$, $f^*(-x) = -f^*(x)$.

\paragraph{Linearity of odd predictors.}
Proposition~\ref{th:learning-anti-inv} ensures that the predictor $f(\mu_t; \cdot)$ associated with the Wasserstein GF of Equation~\eqref{eq:w-gf} is also odd at any time $t \geq 0$, and we can thus write, for any $x$, $f(\mu_t; x) = \frac{1}{2} \left(f(\mu_t; x) -  f(\mu_t; -x) \right)$, which yields 
\begin{align*}
    f(\mu_t; x) = \frac{1}{2} \left( \int_{a,b} a \left[\sigma(b^\top x) - \sigma(-b^\top x) \right] \straightd\mu_t(a,b)  \right) = \frac{1}{2} \int_{a,b} a \left(b^\top x \right) \straightd\mu_t(a,b),
\end{align*}
where the last equality stems from the fact that for ReLU, $\sigma(x) - \sigma(-x) = x$. 
Put differently, the predictor is \textbf{linear}:~it is the same as replacing $\sigma$ by $\frac{1}{2} \text{id}_{\mathbb{R}^d}$, and $f(\mu_t; x) = {w(t)}^\top x$, where 
\begin{align}\label{eq:w(t)}
    w(t) :&= \frac{1}{2} \int_{a,b} a\,b \, \straightd\mu_t(a,b) \in \mathbb{R}^d.
\end{align}
This degeneracy is not surprising as in fact, a linear predictor is the best one can hope for in this setting. Indeed, consider the following assumption and the next lemma:
\begin{assumption}[Squared loss function]\label{ass:squared-loss}
The loss function $\ell$ is the squared loss, \ie, $\ell(y, \hat{y}) = \frac{1}{2} (y - \hat{y})^2$, and thus satisfies the condition of Proposition~\ref{th:learning-anti-inv}.
\end{assumption}
We make this assumption in order to provide an explicit convergence rate in \autoref{th:lin-exp-cv} below.
\begin{lemma}[Optimality of odd predictors]\label{th:opt-lin}
Let $f$ be a predictor in the hypothesis class $\mathcal{F} := \left\{ :x \mapsto \int a\sigma(b^\top x) \straightd \mu(a,b) ; \mu \in \mathcal{P}_2(\mathbb{R} \ \times \mathbb{R}^d)  \right)\}$. Then, denoting $f_{\text{odd}}(x) := \frac{1}{2}(f(x) - f(-x))$ (\resp~$f_{\text{even}} := \frac{1}{2}(f(x) + f(-x))$) the odd (\resp~even) part of $f$, one has:
\begin{align*}
    &(i) \quad f_{\text{odd}} \in \mathcal{F}, \\
    &(ii) \ \  L(f):= \mathbb{E}_{x \sim \rho} \left[ \left(f^*(x) - f(x) \right)^2 \right] \geq \mathbb{E}_{x \sim \rho} \left[ \left(f^*(x) - f_{\text{odd}}(x) \right)^2 \right] =: L(f_{\text{odd}}), \\
    &(iii) \ \, \text{equality holds if and only if } f \text{ is odd} \ \rho\text{-almost surely}.
\end{align*}
\end{lemma}
\begin{proof}
The result readily follows from the decomposition $f = f_{\text{odd}} + f_{\text{even}}$ which leads to
\begin{align*}
    L(f) = L(f_{\text{odd}}) \ + \  \underbrace{\mathbb{E}_{x \sim \rho} \left[\left(f_{\text{even}}(x)\right)^2 \right]}_{\geq 0} \ - \ 2 \underbrace{\mathbb{E}_{x \sim \rho} \left[\left(f^*(x) - f_{\text{odd}}(x) \right)f_{\text{even}}(x) \right]}_{0 \text{ by symmetry}}.
\end{align*}
% The last expectation is $0$ because the integrand is odd as the product of an odd function with an even function, and the measure $\rho$ is symmetric.
We then get that $L(f) \geq L(f_{\text{odd}})$ with equality if and only if $\mathbb{E}_{x \sim \rho} \left[\left(f_{\text{even}}(x)\right)^2 \right] = 0$, \ie, $f_{\text{even}}(x) = 0$ for $\rho$-almost every $x$. Finally, if $\mu \in \mathcal{P}_2(\mathbb{R}^{d+1})$, then $\nu := \frac{1}{2}(\mu + S_{\#}\mu) \in \mathcal{P}_2(\mathbb{R}^{d+1})$, where $S: (a,b) \in \mathbb{R}^{d+1} \mapsto (-a, -b)$, and $f(\nu; \cdot) = f_{\text{odd}}(\mu; \cdot)$, which shows $f_{\text{odd}}(\mu; \cdot) \in \mathcal{F}$.
\end{proof}
Since, as shown above, any odd predictor turns out to be linear because of the symmetries of ReLU, in this context, the best one can expect is thus to learn the best linear predictor.

\paragraph{Exponential convergence for linear networks.}
\looseness=-1
We are thus reduced to studying the dynamics of linear networks (which in our case are infinitely wide), which is an interesting object of study its own right (Ji and Telegarsky~\cite{ji2018gradient} show a result similar to our result below in the finite-width case with the logistic loss on a binary classification task). In this case, the Wasserstein GF~\eqref{eq:w-gf} (with ReLU replaced by $\frac{1}{2} \text{id}_{\mathbb{R}^d}$) is defined for more general input distributions $\mathbb{P} \in \mathcal{P}_2(\mathbb{R}^d)$ (\eg, empirical measures) and target functions $f^*$. The objective in this context is thus to learn:
\begin{equation}\label{eq:Q}
\begin{aligned}
    w^\star \in \argmin_{w \in \mathbb{R}^d} \left( Q(w) := \frac{1}{2} \mathbb{E}_{x \sim \mathbb{P}} \left[\left(f^*(x) - \innerprod{w}{x} \right)^2 \right] \right)
\end{aligned}
\end{equation}
with the dynamics of linear infinitely wide two-layer networks described by the Wasserstein GF~\eqref{eq:w-gf} where the activation function $\sigma$ is replaced by $\frac{1}{2} \text{id}_{\mathbb{R}^d}$.
% Expanding the square, the function $Q$ to minimize is equal to 
% \begin{align*}
%     Q(w) = \frac{1}{2} \left(\mathbb{E}[f^*(x)^2] - 2 \Big \langle w, \, \mathbb{E}[f^*(x)x] \Big \rangle + w^\top \mathbb{E}[x x^\top] w\right),
% \end{align*}
\autoref{th:lin-exp-cv} below shows exponential convergence to a global minimum of $Q$ as soon as the problem is strongly convex. 
Note that although in this case both $\phi(\cdot; \cdot)$ (see Equation~\eqref{eq:predictor}) and the predictor in the objective $Q$ are linear \wrt~the input, only the predictor in $Q$ is linear in the parameters (ordinary least squares).
\begin{theorem}\label{th:lin-exp-cv}
Assume that the smallest eigenvalue $\lambda_{\text{min}}$ of $\mathbb{E}_{x \sim \mathbb{P}}[x x^\top]$ is positive. Let $(\mu_t)_{t \geq 0}$ be the Wasserstein GF associated to~\eqref{eq:w-gf} with activation function $\frac{1}{2} \text{\normalfont{ id}}_{\mathbb{R}^d}$ instead of $\sigma = \text{\normalfont{ ReLU}}$, and call $w(t) = \frac{1}{2} \int a b \, \straightd \mu_t(a,b) \in \mathbb{R}^d$. Then, there exits $\eta > 0$ and $t_0 > 0$ such that, for any $t \geq t_0$,
\begin{align*}
    \Big(Q(w(t)) - Q(w^\star) \Big) \leq e^{-2 \eta \lambdamin (t-t_0)} \Big(Q(w(t_0)) - Q(w^\star) \Big).
\end{align*}
\end{theorem}
\begin{remark}
Note that as soon as $\mathbb{P}$ has spherical symmetry, the problem becomes strongly convex by Lemma~\ref{th:spherical-sym-id}. Note that although $F(\mu_t) = Q(w(t))$, $(w(t))_{t \geq 0}$ is \textbf{not} a gradient flow for the (strongly) convex objective $Q$ (which would immediately guarantee exponential convergence to the global minimum).
\end{remark}
\begin{wrapfigure}{r}{0.35\textwidth}
\vspace{-2.5em}
\centering
    \includegraphics[width=0.38\textwidth]{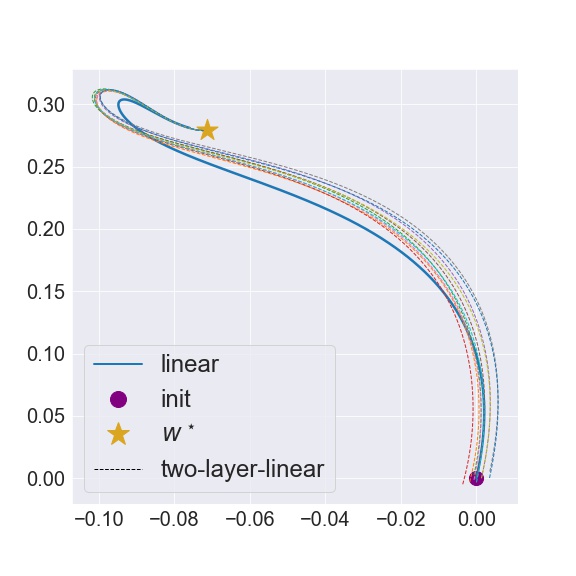}
    \vspace{-2.8em}
    \caption{
    %Two-layer linear network \textit{vs} pure linear model: GD path for two coordinates.
    GD path for two coordinates: two-layer linear network \textit{vs} pure linear model.
    }
\label{fig:lin-vs-2-layer-cv}
\end{wrapfigure}
\noindent
The proof, provided in Appendix~\ref{app:lin-exp-cv}, proceeds in two steps: first it is shown that $w^\prime(t) = - H(t) \nabla Q(w(t))$ for some positive definite matrix $H(t)$ whose smallest eigenvalue is always lower-bounded by a positive quantity, then we prove that this leads to exponential convergence.
Figure~\ref{fig:lin-vs-2-layer-cv} illustrates that the dynamics of GF on $F$ remain non-linear in that they do not reduce to GF on $Q$ (although the paths are close). To simulate GF on $F$ we use a large (but finite) number of neurons $m=1,024$ and a small (but positive) step-size $10^{-2}$ and simply proceed to do GD on the corresponding finite-dimensional objective (see comment in Section~\ref{sec:pb-setting} on relationship between the Wasserstein GF and finite-width GD).

\section{Learning the low-dimensional structure of the problem}\label{sec:low-dim}

Consider a linear sub-space $H$ of dimension $d_H < d$ (potentially much smaller than the ambient dimension), and assume $f^*$ has the following structure: $f^*(x) = f_H(p_H(x))$ where $p_H$ is the orthogonal projection onto $H$ (which we also write $x^H$ for simplicity, and we reserve sub-scripts for denoting entries of vectors) and $f_H:H \rightarrow \mathbb{R}$ is a given function. 

In this context it is natural to study whether the learned function shares the same structure as $f^*$. As observed in Figure~\ref{fig:perp-dependence} this is not the case in finite time, but it is reasonable however to think that the learned predictor $f(\mu_t; \cdot)$ shares the same structure as $f^*$ as $t \to \infty$, and we give numerical evidence in this direction. On the other hand, we prove rigorously that the structure of the problem allows to reduce the dynamics to a lower-dimensional PDE. In this section, we consider for simplicity that $\mu_0^1$ is the uniform distribution over $\{-1,+1\}$ and that $\mu_0^2$ is the uniform distribution over $\mathbb{S}^{d-1}$.

\paragraph{Comment on the assumptions for this section.}
The assumptions that $|a| = ||b||$ on the  support of $\mu_0$ is crucial here. This ensures that the Wasserstein GF~\eqref{eq:w-gf} is well-defined and that $\mu_t$ stays supported on the set $\{|a| = ||b||\}$ for any $t \geq 0$, a fact which is used in the proofs. The assumption that $\rho$ is the uniform distribution over the unit sphere bears some importance but could likely be replaced by other measures with spherical symmetry provided that the dynamics would still be well-defined and at the cost of more technical proofs.

\subsection{Symmetries and invariance}\label{sec:low-dim-gen}

The structure of $f^*$ implies that it is invariant by any $T \in \mathcal{O}(d)$ which preserves $H$, \ie, such that its restrictions to $H$ and $H^\perp$ are $T_{|H} = \text{id}_H$ and $T_{|H^\perp} \in \mathcal{O}(d_\perp)$, where $\mathcal{O}(d_\perp)$ is the orthogonal group of $H^\perp$ whose dimension is $d_\perp = d - d_H$. By Proposition~\ref{th:learning-inv}, such transformations also leave the predictor $f(\mu_t; \cdot)$ invariant for any $t \geq 0$ since $\rho$ is spherically symmetric. Lemma~\ref{th:inv-sub-orthogonal} below then ensures that $f(\mu_t; x)$ depends on the projection $x^\perp$ onto $H^\perp$ only through its norm, that is $f(\mu_t; x) = \Tilde{f}_t(x^H, ||x^\perp||)$ for some $\Tilde{f}_t : H \times \mathbb{R}_+ \rightarrow \mathbb{R}$.
\begin{lemma}[Invariance by a sub-group of $\mathcal{O}(d)$]\label{th:inv-sub-orthogonal}
Let $f:\mathbb{R}^d \rightarrow \mathbb{R}$ be invariant under any $T \in \mathcal{O}(d)$ such that $T_{|H} = \normalfont{\text{id}}_H$ and $T_{|H^\perp} \in \mathcal{O}(d_\perp)$. Then, there exists some $\Tilde{f}: H \times \mathbb{R}_+ \rightarrow \mathbb{R}$ such that for any $x \in \mathbb{R}^d$, $f(x) = \Tilde{f}(x^H, ||x^\perp||)$.
\end{lemma}
\begin{proof}
Consider $\Tilde{f}:(x^H, r) \in H \times \mathbb{R}_+ \mapsto f(x^H + r e^\perp_1)$ where $e^\perp_1$ is the first vector of an orthonormal basis of $H^\perp$, and let $x \in \mathbb{R}^d$. If $x^\perp = 0$, the result is obvious. Otherwise, consider an orthogonal linear map $T_x$ such that ${T_x}_{|H} = \normalfont{\text{id}}_H$ and $T_x$ sends $x^\perp / ||x^\perp||$ on $e^{\perp}_{1}$. The invariance of $f$ under $T_x$ implies $f(x) = f(T_x(x)) = f(x^H + ||x^\perp|| e^\perp_1) = \Tilde{f}(x^H, ||x^\perp||)$.
\end{proof}

Figure~\ref{fig:perp-dependence} shows that the dependence in $||x^\perp||$ cannot be removed in finite time: $f(\mu_t; u_H +  r e^{\perp}_{1})$ does depend on the distance $r \in \mathbb{R}_+$ to $H$, but this dependence tends to vanish as $t \to \infty$. The plots of Figure~\ref{fig:perp-dependence} are obtained by discretizing the initial measure $\mu_{0,m} = \frac{1}{m} \sum_{j=1}^m \delta_{(a_j(0), b_j(0))}$ with $m=1,024$ atoms, and sampling $a_j(0) \sim \mathcal{U}(\{-1, +1\})$ and $b_j(0) \sim \mathcal{U}(\mathbb{S}^{d-1})$. We perform GD with a finite step-size $\eta=$ and a finite number $n=256$ of fresh \iid~samples from the data distribution per step with $f^*(x) = ||x^H||$, $d=20$ and $d_H=5$.

\begin{figure}[!htb]
\centering
    \subfloat[$f(\mu_t; u_H + r e^{\perp}_{1})$ \textit{vs} $r$]{{\includegraphics[width=0.3\linewidth]{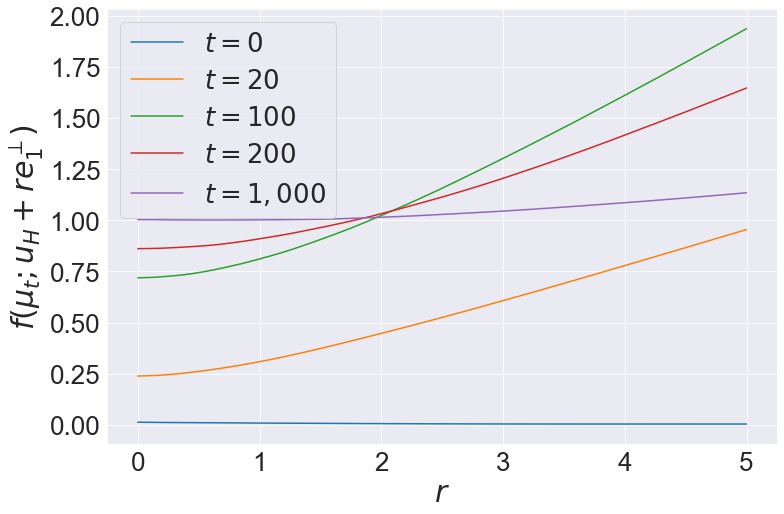}}}
    \quad
    \centering
    \subfloat[$f(\mu_t; u_H +  r e^{\perp}_{1})$ \textit{vs} $t$]{{\includegraphics[width=0.3\linewidth]{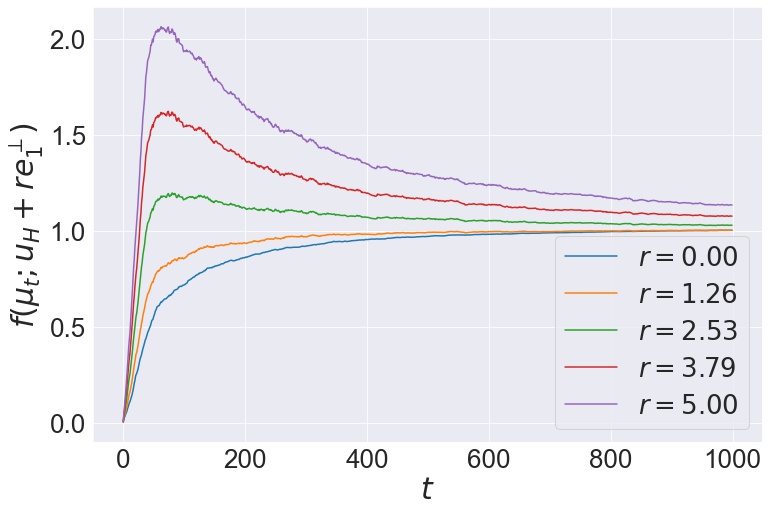} }}
    \caption{$f(\mu_t; u_H + r e^{\perp}_{1})$ \textit{vs} $r$ and $t$ for a random $u_H \in \mathbb{S}^{d_H-1}$ with $d=20$, $d_H=5$.}
\label{fig:perp-dependence}
\end{figure}
%In addition, recall that Proposition~\ref{th:learning-inv} also states that $\mu_t$ has similar symmetries which suggest that the dynamics of $\mu_t$ reduce to simpler dynamics on a measure over $d_H + 1$ variables, which is precisely what we show below.
\paragraph{Dynamics over the sphere $\mathbb{S}^{d-1}$.}
Using the positive $1$-homogeneity of ReLU, and with the assumptions on $\mu_0$, the dynamics on $\mu_t \in \mathcal{P}_2(\mathbb{R}^{d+1})$ can be reduced to dynamics on the space $\mathcal{M}_+(\mathbb{S}^{d-1})$ of non-negative measures over $\mathbb{S}^{d-1}$: only the direction of neurons matter and their norm only affects the total mass. From  this point of view, neurons with positive and negative output weights behave differently and have separate dynamics. Indeed, consider the pair of measures $(\nu_t^+, \nu_t^-) \in \mathcal{M}_+(\mathbb{S}^{d-1})^2$ characterized by the property that for any continuous test function $\varphi: \mathbb{S}^{d-1} \rightarrow \mathbb{R}$,
\begin{align}\label{eq:nu-pm}
    \int_u \varphi(u) \straightd \nu_t^\pm(u) = \int_{\pm a \geq 0, b} |a| \, ||b|| \varphi \left( \frac{b}{||b||} \right) \straightd \mu_t(a,b),
\end{align}
where we have used the superscript $^\pm$ to denote either or $\nu_t^+$ or $\nu_t^-$ and the right-hand side is changed accordingly (the integration domain) depending on the sign $+$ or $-$. Because ReLU is positively $1$-homogeneous, we have $f(\mu_t; x) = \int \sigma(u^\top x) \straightd(\nu_t^+ - \nu_t^-)(x)$. It is shown in Appendix~\ref{app:low-dim-nu} that $\nu_t^\pm$ satisfies, in the sense of distributions, the equation
\begin{align}\label{eq:nu_t}
    \partial_t \nu_t^{\pm} =  - \text{div} \left(\pm \Tilde{v}_t \nu_t^\pm \right) \pm 2 g_t \nu_t^\pm,
\end{align}
where, for any $u \in \mathbb{S}^{d-1}$,
\begin{equation}\label{eq:advec-reac-general}
    \begin{aligned}
        g_t(u) &= -\int_y \partial_2 \ell \Big(f^*(y), \,  f(\mu_t; y) \Big) \sigma(u^\top y) \straightd\rho(y), \\
        \Tilde{v}_t(u) &= - \int_y \partial_2 \ell \Big(f^*(y), \,  f(\mu_t; y) \Big) \sigma^\prime(u^\top y) \left[y - (u^\top y) u\right] \straightd\rho(y).
    \end{aligned}
\end{equation}
Equation~\eqref{eq:nu_t} can be interpreted as a Wasserstein-Fisher-Rao GF~\cite{gallouet2019unbalanced} on the sphere since $\Tilde{v}_t(u) = \text{proj}_{\{u\}^\perp}(\nabla g_t(u))$. 

\paragraph{Closed dynamics over $[0, \pi/2] \times \mathbb{S}^{d_H-1}$.}
The dynamics on the pair $(\nu_t^+, \nu_t^-)$ can be further reduced to dynamics over $[0, \pi/2] \times \mathbb{S}^{d_H-1}$. Indeed, by positive 1-homogeneity of $f(\mu_t; \cdot)$ we may restrict ourselves to inputs $u \in \mathbb{S}^{d-1}$, and $f(\mu_t; u)$ depends only on $u^H$ and $||u^\perp||$. However, because $||u^H||^2 + ||u^\perp||^2 = 1$, this dependence translates into a dependence on the direction $u^H / ||u^H||$ of the projection onto $H$ and the norm $||u^H||$. The former is an element of $\mathbb{S}^{d_H-1}$ while the latter is given by the angle $\theta$ between $u$ and $H$, that is $\theta := \arccos(u^\top u^H/||u^H||) = \arccos(||u^H||)$. This simplification leads to the following lemma:

\begin{lemma}\label{th:red-d_H}
Define the measures $\tau_t^+, \tau_t^-$ by $\tau_t^\pm = P_{\#} \nu_t^\pm \in \mathcal{M}_{+}([0, \pi/2] \times \mathbb{S}^{d_H-1})$ via $P: u \in \mathbb{S}^{d-1} \backslash H^\perp \mapsto (\arccos(||u_H||), u_H/||u_H||) \in [0, \pi/2] \times \mathbb{S}^{d_H-1}$. Then, the measures $\tau_t^+, \tau_t^-$ satisfy the equation
\begin{align}\label{eq:tau_t-general}
    \partial \tau_t^\pm &= -\text{\normalfont{div}} \left(\pm V_t \tau_t^\pm \right) \pm 2 G_t \tau_t^\pm,
\end{align}
where $G_t: [0, \pi/2] \times \mathbb{S}^{d_H-1} \rightarrow \mathbb{R}$, and $V_t: [0, \pi/2] \times \mathbb{S}^{d_H-1} \rightarrow \mathbb{R}^{d_H+1}$ are functions depending \textbf{only} on $(\tau_t^+, \tau_t^-)$, and furthermore, $f(\mu_t; \cdot)$ can be expressed solely using $\tau_t^+, \tau_t^-$ (exact formulas are provided in Appendix~\ref{app:low-dim-tau}).
\end{lemma}
Abbe et.~al~\cite{abbe2022staircase} show a similar result with a lower-dimensional dynamics in the context of infinitely wide two-layer networks when the input data have i.i.d~coordinates distributed uniformly over $\{-1, +1\}$ (\ie, Rademacher variables), except that they do not have the added dimension due to the angle $\theta$ as we do thanks to their choice of input data distribution.

Lemma~\ref{th:red-d_H} above illustrates how the GF dynamics of infinitely wide two-layer networks adapts to the lower-dimensional structure of the problem: the learned predictor and the dynamics can described only in terms of the angle $\theta$ between the input neurons and $H$ and their projection on the unit sphere of $H$. 

%shows that when $f^*$ only depends on the orthogonal projection on $H$ the dynamics over the initial parameters $(a,b) \in \mathbb{R}^{1+d}$ reduces to dynamics over the angle $\theta$

\subsection{One dimensional reduction}\label{sec:1d-red}

Since the predictors we consider are positively homogeneous, one cannot hope to do better than learn a positively homogeneous function. A natural choice of such a target function to learn is the Euclidean norm. With the additional structure that the target only depends on the projection onto $H$, this leads to considering $f^*(x) = ||x^H||$ which has additional symmetries compared to the general case presented above: it is invariant by any linear map $T$ such that $T_{|H} \in \mathcal{O}(d_H)$ and $T_{|H^\perp} \in \mathcal{O}(d_\perp)$. By Proposition~\ref{th:learning-inv} those symmetries are shared by $\mu_t$ and $f(\mu_t; \cdot)$, and we show that in this case the dynamic reduces to a one-dimensional dynamic over the angle $\theta$ between input neurons and $H$. 

We prove a general disintegration result for the uniform measure on the sphere in the Appendix (see Lemma~\ref{th:disintegration}) which allows, along with some spherical harmonics analysis, to describe the reduced dynamics and characterize the objective that they optimize. This leads to the following result:

\begin{theorem}[1d dynamics over the angle $\theta$]\label{th:1d-theta}
Assume that $f^*(x) = ||x^H||$, and define the measures $(\tau_t^+, \tau_t^-) \in {\mathcal{M}_+([0, \pi/2])}^2$ from $(\nu_t^+, \nu_t^-)$ via $P: u \in \mathbb{S}^{d-1} \mapsto \arccos(||u_H||) \in [0, \pi/2]$: $\tau_t^\pm = P_{\#} \nu_t^\pm$. Then, the pair $(\tau_t^+, \tau_t^-)$ follows the Wasserstein-Fisher-Rao GF for the objective $A(\tau^+, \tau^-) := \mathbb{E} \left[\ell \left(f(\tau^+, \tau^-; x), f^*(x)\right) \right]$ over the space $\mathcal{M}_+([0, \pi/2]) \times \mathcal{M}_+([0, \pi/2])$, where $f(\tau^+, \tau^-; x)$ is the expression (with a slight overloading of notations) of $f(\mu;x)$ in function of $(\tau^+, \tau^-)$ (see Appendix~\ref{app:1d} for more details):
\begin{align}
    \straightd \tau_0^{\pm}(\theta) &= \frac{1}{B \left(\frac{d_H}{2}, \frac{d_\perp}{2} \right)} \cos(\theta)^{d_H-1} \sin(\theta)^{d_\perp-1} \straightd \theta, \nonumber \\
    \partial_t \tau_t^\pm &= -{\normalfont\text{div}} \left(\pm V_t \tau_t^\pm \right) \pm 2 G_t \tau_t^\pm, \label{eq:1d-WFR}
\end{align}
where $B$ is the Beta function, and
\begin{align*}
    G_t(\theta) &= - \int_y \partial_2 \ell \Big(f^*(y), \,  f(\mu_t; y) \Big) \sigma \left(\cos(\theta) y^H_1 + \sin(\theta)y^\perp_1 \right) \straightd \rho(y), \\  
    V_t(\theta) &= G_t'(\theta).
\end{align*}
Additionally, $f(\mu_t; \cdot)$, $G_t$, and $V_t$ \textbf{only depend} on the pair $(\tau_t^+, \tau_t^-)$, and for any $t \geq 0$, it holds that $F(\mu_t) = A(\tau_t^+, \tau_t^-)$.
\end{theorem}

\begin{remark}

The result should still hold for general $\rho$ which are spherically symmetric as long as the Wasserstein GF~\eqref{eq:w-gf} is well-defined but the proof is more technical. In addition, this result shows that even with more structure than in Lemma~\ref{th:red-d_H}, the dynamics of infinitely wide two-layer networks are still able to adapt to this setting: these dynamics, as well as the learned predictor, can be \textit{fully characterized} solely by the one-dimensional dynamics over the angle $\theta$ between input neurons and $H$.
This is noteworthy since this angle determines the alignment of the neurons with $H$, and thus measures how much the representations learned by the network have adapted to the structure of the problem. Furthermore, as discussed below, this reduction with exact formulas enables efficient numerical simulation in one dimension.
\end{remark}

Daneshmand and Bach~\cite{daneshmand2022polynomial} prove the global convergence of a reduced one-dimensional dynamics in a context similar to ours but their original problem is two-dimensional and with a  choice of activation function that leads to specific algebraic properties.

\paragraph{Expression of $f(\mu_t; \cdot)$.}

Because of the symmetries of $f(\mu_t; \cdot)$, which result from that of $f^*$, $f(\mu_t; x)$ depends only on $||x^H||$ and $||x^\perp||$. What is more, since $f(\mu_t; \cdot)$ is positively $1$-homogeneous (because ReLU is) it actually holds that $f(\mu_t; x) = ||x|| \Tilde{f}_t(\varphi_x)$ where $\varphi_x = \arccos(||x^H|| / ||x||)$ is the angle between $x$ and $H$, and $\Tilde{f}_t(\varphi) := \int_{\theta} \Tilde{\phi}(\theta; \varphi) \straightd(\tau_t^+ - \tau_t^-)(\theta)$, $\Tilde{\phi}$ depending only on $\sigma$ and fixed probability measures (see Appendix~\ref{app:1d-WFR-GF} for an exact formula).

\paragraph{Learning the low-dimensional structure as $t \to \infty$.}
Although, as shown in Figure~\ref{fig:perp-dependence}, $f(\mu_t;\cdot)$ does not learn the low-dimensional structure in finite-time, it is reasonable to expect that as $t \to \infty$, the measures $\tau_t^\pm$ put mass only on $\theta=0$, indicating that the only part of the space that the predictor is concerned with for large $t$ is the sub-space $H$. Since we assume here that the target function $f^*$ is non-negative, the most natural limits for $\tau_t^+$ and $\tau_t^-$ are $\tau_t^+ \to \alpha \delta_{0}$ with $\alpha > 0$, and $\tau_t^- \rightarrow 0$ (in the sense that $\tau_t^-([0, \pi/2]) \rightarrow 0$) as $t \rightarrow \infty$, because then the \quoting{negative} output weights do not participate in the prediction in the large $t$ limit. 

The global convergence result of Chizat and Bach~\cite{chizat2018global, wojtowytsch2020convergence} still holds but is not quantitative and moreover does not guarantee that the limit is the one described above. We leave the proof of this result as an open problem, but we provide numerical evidence supporting this conjecture.
Indeed, we take advantage of the one-dimensional reduction from \autoref{th:1d-theta}, and numerically simulate the resulting dynamics by parameterizing $\tau_t^\pm$ via weight and position~\cite{chizat2022sparse} as $\mu_{m,t} = (1/m) \sum_{j=1}^m c_j^\pm(t) \delta_{\theta_j^\pm(t)}$, and simulating the corresponding dynamics for $c_j^\pm(t)$ and $\theta_j^\pm(t)$. 
The corresponding results are depicted in Figure~\ref{fig:dist-total-var} which are again obtained by discretizing the initial measures $\tau_0^+, \tau_0^-$ and performing GD with finite step-size (see more details in Appendix~\ref{app:numerical}).
Figures~\ref{fig:plus-dist} and~\ref{fig:minus-dist} show that the mass of $\tau_t^+$ tends to concentrate around $0$ while that of $\tau_t^-$ tends to concentrate around $\pi/2$, indicating that $\tau_t^+$ adapts to the part of the space relevant to learning $f^*$ while $\tau_t^-$ puts mass close to the orthogonal to that space. 

\paragraph{Total mass of particles at convergence.}
If $\tau_\infty^- = 0$ and $\tau_\infty^+ = \alpha \delta_{0}$ as described above, we have $f(\mu_\infty; x) = \alpha ||x|| \Tilde{\phi}(0; \varphi_x) = \alpha \frac{\Gamma(d_H/2)}{2 \sqrt{\pi} \Gamma((d_H+1)/2)} ||x|| \cos(\varphi_x) =  \frac{\alpha \Gamma(d_H/2)}{2 \sqrt{\pi} \Gamma((d_H+1)/2)} ||x^H||$. To recover exactly $f^*$, it must hold that $\alpha = \tau_\infty^+([0, \pi/2]) = \frac{2 \sqrt{\pi} \Gamma((d_H+1)/2)}{\Gamma(d_H/2)}$. 
Defining the normalized probability measure $\tilde{\tau}_t^\pm = \tau_t^\pm / \tau_t^\pm([0, \pi/2])$, we thus expect $\tilde{\tau}_t^+$ to grow close to $\delta_{0}$ and $\tilde{\tau}_t^-$ to $\delta_{\pi/2}$. In terms of total mass, we expect that $\tau_t^+([0,\pi/2])$ gets closer to $\alpha$ while $\tau_t^-([0,\pi/2])$ gets closer to $0$. 

The numerical behaviour depicted in Figure~\ref{fig:pos-mass-dist} seems to follow our intuitive description, at least until a critical time $t^*$ in the numerical simulation which corresponds to the first time $t$ where $\tau_t^+([0, \pi/2]) > \alpha$. While the total mass of $\tau_t^\pm$ (dashed lines) seems to approach its limit rapidly before $t^*$ it slowly  moves further away from it for $t \geq t^*$. On the other hand, while the angles only slowly change before $t^*$, they start converging fast towards the corresponding Dirac measures after $t^*$. It is unclear whether this slight difference in behaviour (around the critical time $t^*$) between what we intuitively expected and the numerical simulation is an artefact of the finite width and finite step size or if it actually corresponds to some phenomenon present in the limiting model. For more details concerning the numerical experiments, see Appendix~\ref{app:numerical}.
\begin{figure}[!htb]
\centering
    \subfloat[$\tau_t^+$ distributions]{\label{fig:plus-dist}{\includegraphics[width=0.3\linewidth]{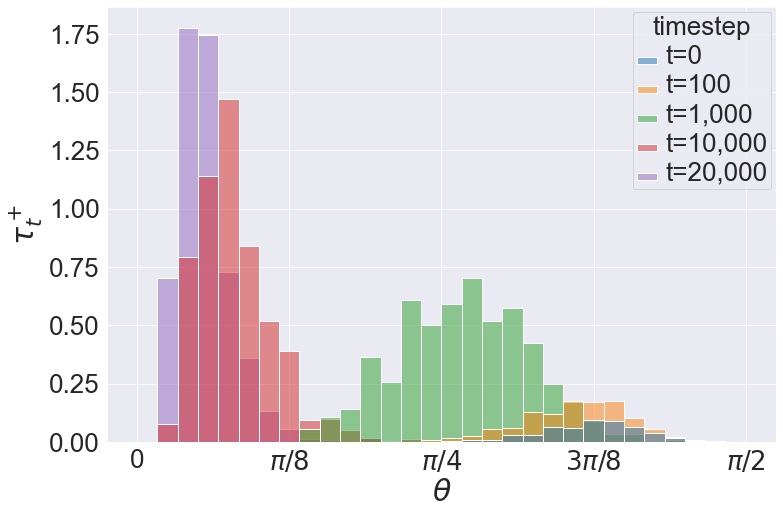}}}
    \quad
    \centering
    \subfloat[$\tau_t^-$ distributions]{\label{fig:minus-dist}{\includegraphics[width=0.3\linewidth]{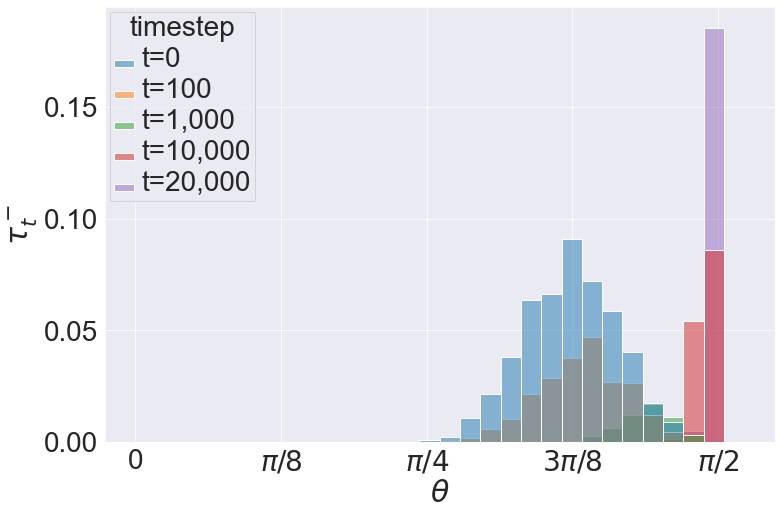} }}
    \quad
    \centering
    \subfloat[Position / mass distances ]{{\label{fig:pos-mass-dist} \includegraphics[width=0.3\linewidth]{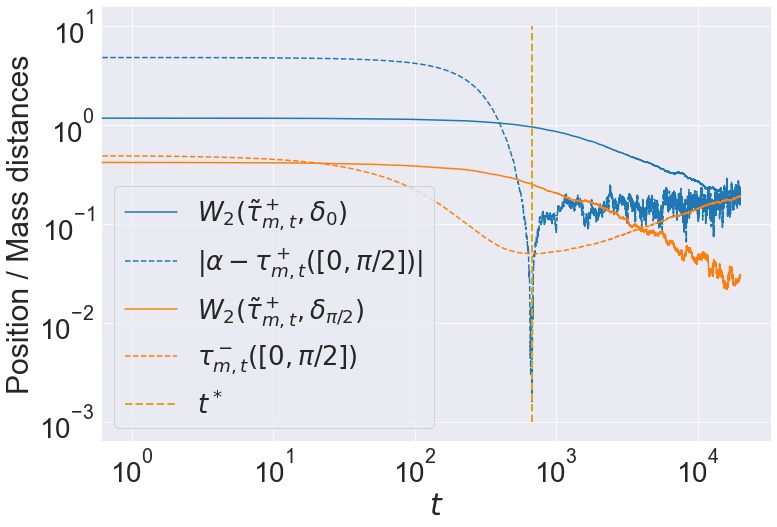}}}
    %\qquad
    \caption{Angle distributions $\tau_t^+ / \tau_t^-$ and position / mass distances with $m=1024$, $d=30$ and $d_H=5$. (a) (\textit{resp.}~(b)) $\tau_t^+$ (\textit{resp.}~$\tau_t^-$) as a histogram for different $t$.  $(c)$ distances (in log-log scales) of the mass and positions of positive (blue) / negative (orange) particles to the intuitively expected limits: the distance in position is the Wasserstein-2 distance of the normalized (probability) measures $\tilde{\tau}_t^\pm$ to the corresponding Dirac measures while the distance in mass is the absolute error to the expected mass as $t \to \infty$.}
\label{fig:dist-total-var}
\end{figure}
Note that there is \textit{a priori} not a unique global optimum: $\tau_\infty^+$ and $\tau_\infty^-$ (if they exist) can compensate on parts of the space $[0, \pi/2]$ and lead to the same optimal predictor for different choices of measures. Our numerical experiments suggest that the GF dynamics select a \quoting{simple} solution where $\tau_\infty^+$ is concentrated on $\{\theta=0\}$ and $\tau_\infty^-$ vanishes (puts $0$ mass everywhere), which is a form of implicit bias.

\section{Conclusion}\label{sec:conclusion}
We have explored the symmetries of infinitely wide two-layer ReLU networks and we have seen that: $(i)$ they adapt to the orthogonal symmetries of the problem, $(ii)$ they reduce to the dynamics of a linear network in the case of an odd target function and lead to exponential convergence, and $(iii)$ when the target function depends only on the orthogonal projection onto a lower-dimensional sub-space $H$, the dynamics can be reduced to a lower-dimensional PDE. In particular, when $f^*$ is the Euclidean norm, this PDE is over a one-dimensional space corresponding to the angle $\theta$ between the particles and $H$. We have presented numerical experiment indicating that the positive particles converge to the subspace $H$ in this case and leave the proof of this result as an open problem.
We also leave as an open question whether the results of Section~\ref{sec:sym} extend to deeper networks.

\section*{Acknowledgments}
Karl Hajjar received full support from
the Agence Nationale de la Recherche (ANR), reference ANR-19-CHIA-0021-01 “BiSCottE". The authors thank Christophe Giraud for useful discussions and comments.

\bibliography{biblio}

\appendix

\section*{Appendix}\label{app:appendix}

\section{Additional notations and preliminary results}

\subsection{Notations for the appendix}\label{app:notations}

We introduce in this section additional notation that we use throughout the Appendix.

\paragraph{Residual:} we call $R_t(y) := - \partial_2 \ell(f^*(y), f(\mu_t; y))$, the \quoting{residual}, which is equal to the difference $f^*(y) - f(\mu_t; y)$ when $\ell$ is the squared loss.

\paragraph{Identity matrix:} we denote by $I_p$ the identity matrix in $\mathbb{R}^{p \times p}$ for any $p \in \mathbb{N}$.

\paragraph{Indicator functions:} we denote by $\mathbf{1}_{A}$ the indicator of a set $A$, that is $\mathbf{1}_{A}(z) = 1 \iff z \in A$, and $\mathbf{1}_{A}(z) = 0$ otherwise.

\paragraph{Total variation:} for any measure $\nu$, we denote by $|\nu|$ its total variation, which should cause no confusion with the absolute value given the context.

\paragraph{Beta / Gamma function and distribution:} for $\alpha, \beta > 0$, we denote by $B(\alpha, \beta)$ the Beta function equal to $\Gamma(\alpha) \Gamma(\beta) / \Gamma(\alpha + \beta)$ where $\Gamma$ is the Gamma function, and by $\text{Beta}(\alpha, \beta)$ the beta law with density equal to $u^{\alpha-1} u^{\beta-1} / B(\alpha, \beta)$ on $[0,1]$.

\paragraph{Gaussian / spherical measures:} we call $\rho_p$ the standard Gaussian measure in $\mathbb{R}^p$ (corresponding to $\mathcal{N}(0, I_p)$) for any $p \in \mathbb{N}$. 

Whenever $\tau \in \mathcal{M}_+(\Omega)$ has finite and non-zero total variation, we denote by $\Tilde{\tau} \in \mathcal{P}_2(\Omega)$ its normalized counterpart (which is a probability measure), that is $\Tilde{\tau} = \tau / \tau(\Omega) = \tau  / |\tau|$. 

For any $p \in \mathbb{N}$, we call $\omega_p$ the Lebesgue (spherical) measure over the unit sphere  $\mathbb{S}^{p-1}$ of $\mathbb{R}^p$, that is the measure such that $\Tilde{\omega}_p$ is the \textit{uniform} measure on $\mathbb{S}^{p-1}$. We then denote by $|\mathbb{S}^{p-1}|$ the surface area of $\mathbb{S}^{p-1}$, that is $|\mathbb{S}^{p-1}| := |\omega_p| = \omega_p(\mathbb{S}^{p-1}) = 2\pi^{p/2} / \Gamma(p/2)$.

\paragraph{Smooth functions:} we denote by $\mathcal{C}(\Omega)$ (\resp~$\mathcal{C}^1_c(\Omega)$) the set of continuous (\resp~continuously differentiable and compactly supported) functions from a set $\Omega$ to $\mathbb{R}$. 

\subsection{General results on invariance for measures and functions}\label{app:prelim-gen}

In this section, we list a number of lemmas related to symmetries of measures and functions which will prove helpful in the proofs presented in the Appendix.

\begin{lemma}[Invariance under invertible maps]\label{th:inv-invertible}
Let $\mu$ be a measure invariant under some measurable and invertible map $T$. Then, assuming $T^{-1}$ is also measurable, one has that $\mu$ is also invariant under $T^{-1}$.
\end{lemma}
\begin{remark}
A similar result holds for a function $f$ invariant under an invertible map.
\end{remark}
\begin{proof}
Because $\mu$ is invariant under $T$, we have for any measurable set $A$, $\mu(A) = \mu(T^{-1}(A))$. Since $T^{-1}$ is assumed to be measurable, for any measurable set $A$, $T(A)$ is also measurable ($T(A) = (T^{-1})^{-1}(A)$) and thus $\mu(T(A)) = \mu(T^{-1}(T(A)) = \mu(A)$ which shows $\mu$ is invariant under $T^{-1}$.
\end{proof}

% \begin{lemma}[Invariance by a sub-group of $\mathcal{O}(d)$]\label{th:inv-sub-orthogonal}
% Let $\mathcal{O}_H$ be the sub-group of $\mathcal{O}(d)$ composed of all the orthogonal linear maps $T$ such that their restrictions to $H$ and $H^\perp$ are $T_{|H} = \normalfont{\text{id}}_H$ and $T_{|H^\perp} \in \mathcal{O}(d_\perp)$. Then, if $f:\mathbb{R}^d \rightarrow \mathbb{R}$ is invariant under any $T \in \mathcal{O}_H$, there exists some $\Tilde{f}: H \times \mathbb{R}_+ \rightarrow \mathbb{R}$ such that for any $x \in \mathbb{R}^d$, $f(x) = \Tilde{f}(x^H, ||x^\perp||)$.
% \end{lemma}
% \begin{proof}
% Consider $\Tilde{f}:(x^H, r) \in H \times \mathbb{R}_+ \mapsto f(x^H + r e^\perp_1)$ where $e^\perp_1$ is the first vector of the canonical orthonormal basis of $H^\perp$. Now let $x \in \mathbb{R}^d$. If $x^\perp = 0$, then $x=x^H$ and $f(x) = \Tilde{f}(x^H, 0)$. If $x^\perp \neq 0$, consider an orthogonal linear map $T_x \in \mathcal{O}(d)$ such that $T_{|H} = \normalfont{\text{id}}_H$ and $T_{|H^\perp}$ sends some orthonormal basis $(b_1, \ldots, b_{d_\perp})$ of $H^\perp$ with $b_1 = x^\perp / ||x^\perp||$ on the canonical orthonormal basis $(e^\perp_1, \ldots, e^\perp_{d_\perp})$ of $H^\perp$. Because $f$ is invariant under $T_x$ by assumption, one has that $f(x) = f(T_x(x)) = f(x^H + ||x^\perp|| e^\perp_1) = \Tilde{f}(x^H, ||x^\perp||)$.
% \end{proof}

\begin{lemma}[Invariance of the density]\label{th:inv-density}
Let $\nu$ be a measure with density $p$ \wrt~some measure $\mu$, and assume both $\nu$ and $\mu$ are $\sigma$-finite and invariant under some measurable and invertible map $T$, whose inverse $T^{-1}$ is also measurable. Then $p$ is also invariant under $T$ $\mu$-almost everywhere, \ie, $p(T(x)) = p(x)$ for $\mu$-almost every $x$.
\end{lemma}
\begin{proof}
For any measurable $\varphi$ (\wrt~$\mu$, and thus \wrt~$\nu$ as well), $\varphi \circ T^{-1}$ is also measurable, and we have, on the one hand
\begin{align*}
    \int \varphi \circ T^{-1} \straightd \nu = \int \left(\varphi \circ T^{-1} \right) p \, \straightd \mu = \int \varphi \left(p \circ T \right) \straightd \mu,
\end{align*}
and on the other hand
\begin{align*}
    \int \varphi \circ T^{-1} \straightd \nu = \int \varphi \straightd \nu = \int \varphi \, p \, \straightd \mu,
\end{align*}
which shows that $\int \varphi \left(p \circ T \right) \straightd \mu = \int \varphi \, p \, \straightd \mu$, and thus that $p \circ T = p$ $\mu$-almost everywhere.
\end{proof}

\begin{lemma}[Projected variance with spherical symmetry]\label{th:spherical-sym-id}
Let $\zeta$ be a spherically symmetric measure on $\mathbb{R}^p$ (\ie, such that for any orthogonal linear map $T \in \mathcal{O}(p)$, $T_{\#} \zeta = \zeta$),   with finite second moment. Then we have the following matrix identity:
\begin{align*}
    \int_{z} z z^\top \straightd \zeta(z) = v_{\zeta} I_{p}, \qquad
    v_{\zeta} := \int_{z} (z_1)^2 \straightd \zeta(z) = \frac{1}{p} \int_z ||z||^2 \straightd \zeta(z).
\end{align*}
\end{lemma}
\begin{proof}
The $(i,j)$-th entry of the matrix on the left-hand-side is $\int_{z} z_i z_j \straightd \zeta(z)$, and it is readily seen that the terms outside the diagonal are $0$. Indeed, let $(i,j) \in [1,p]^2$ with $i \neq j$, and consider the orthogonal map $T_j: z \in \mathbb{R}^p \mapsto (z_1, \ldots, z_{j-1}, -z_j, z_{j+1}, \ldots, z_p)^\top$. The spherical symmetry of $\rho$ implies that it is invariant under $T_j$, which yields $\int_z z_i z_j \straightd \rho(z) = - \int_z z_i z_j \straightd \rho(z)$, thereby showing that the latter is $0$. To see that the diagonal terms are all equal, it suffices to consider the orthogonal map $S_{i}$ which swaps the 1st and $i$-th coordinates of a vector $z$. The invariance of $\rho$ under $S_i$ yields $\int_z (z_1)^2 \straightd \rho(z) = \int_z (z_i)^2 \straightd \rho(z)$, which concludes the proof.
\end{proof}

\subsection[]{A disintegration result on the unit sphere $\mathbb{S}^{d-1}$}\label{app:prelim-disint}

Consider a $u \in \mathbb{S}^{d-1}$. $u$ is determined by: $(i)$ its angle $\theta := \arccos(||u^H||) \in [0, \pi/2]$ with $H$ (\ie, its angle with its projection $u^H$ onto $H$), $(ii)$ the direction $z^H = u^H / ||u^H|| \in \mathbb{S}^{d_H-1}$ of its projection $u^H$ onto $H$, and finally $(iii)$ the direction $z^\perp = u^\perp / ||u^\perp|| \in \mathbb{S}^{d_\perp-1}$ of its projection $u^\perp$ onto $H^\perp$. Since $||u^H||^2 + ||u^\perp||^2 = 1$, the angle $\theta$ gives both the norms of the projections onto $H$ and $H^\perp$: $||u^H|| = \cos(\theta)$ and $||u^\perp|| = \sin(\theta)$.

When $z$ ranges over the unit sphere $\mathbb{S}^{d-1}$, the angle $\theta$ and the directions $z^H, z^\perp$ range over $[0,\pi/2]$, $\mathbb{S}^{d_H-1}$, and $\mathbb{S}^{d_\perp - 1}$ respectively. We wish to understand what measures we obtain on these three sets when $z$ is distributed on the sphere according to the Lebesgue measure $\omega_d$. We show below below that after the change of coordinates described above (from $u \in \mathbb{S}^{d-1}$ to $(\theta, z^H, z^\perp) \in [0,\pi/2] \times \mathbb{S}^{d_H-1} \times \mathbb{S}^{d_\perp - 1}$), the corresponding measures over $\mathbb{S}^{d_H-1}$ and $\mathbb{S}^{d_\perp-1}$ are uniform measures and the measure over $\theta$ is given by a push-forward of a Beta distribution as defined below:
\begin{definition}[Distribution $\gamma$ of the angle $\theta$]\label{def:gamma}
We define the measure $\gamma$ on $[0, \pi/2]$ with the following density \wrt~the Lebesgue measure on $[0, \pi/2]$:
\begin{align*}
    \straightd \gamma(\theta) := \cos(\theta)^{d_H-1} \sin(\theta)^{d_\perp-1} \straightd \theta.
\end{align*}
\end{definition}
\begin{remark}
$\gamma$ is in fact simply given by $(\arccos \circ \sqrt{\cdot})_{\#} \text{Beta}(d_H/2, d_\perp/2)$. Note that the total variation of gamma is $|\gamma| = \gamma([0, \pi/2]) = \frac{1}{2} B(\frac{d_H}{2}, \frac{d_\perp}{2})$, and the corresponding normalized (probability) measure is $\straightd \tildeGamma(\theta) = \straightd \gamma(\theta) / |\gamma| = \frac{2}{B(\frac{d_H}{2}, \frac{d_\perp}{2})} \cos(\theta)^{d_H-1} \sin(\theta)^{d_\perp-1} \straightd \theta$.
\end{remark}
\noindent
We now state the disintegration theorem and give its proof:
\begin{theorem}[Disintegration of the Lebesgue measure on the sphere]\label{th:disintegration}
Let $\omega_d$ denote the Lebesgue measure on the sphere measure on the sphere of $\mathbb{R}^d$, and let $\gamma$ be the measure of Definition~\ref{def:gamma}. Then, one has
\begin{align*}
    \omega_d = \Phi_{\#} (\omega_{d_H} \otimes \omega_{d_\perp} \otimes \gamma)
\end{align*}
where
\begin{align*}
    \Phi :&\, [0, \pi/2] \times \mathbb{S}^{d_H-1} \times \mathbb{S}^{d_\perp-1}  \rightarrow \, \, \mathbb{S}^{d-1} \\
    &\, (\theta, \qquad z_H, \qquad z_\perp) \qquad \,  \mapsto \, \, \cos(\theta) z_H + \sin(\theta) z_\perp.
\end{align*}
\end{theorem}
\begin{proof}
Denoting $\Tilde{\omega}_d$ the uniform measure on the sphere, $|\mathbb{S}^{d-1}| := \frac{2 \pi^{d/2}}{\Gamma(d/2)}$ the surface are of the sphere in dimension $d$, and $\rho_p$ the standard Gaussian distribution in $\mathbb{R}^p$ for any $p$. Using the well-known fact that $\Tilde{\omega}_d = \Pi_{\#} \rho_d$ with $\Pi: x \in \mathbb{R}^d \backslash \{0\} \mapsto x/||x|| \in \mathbb{S}^{d-1}$, we have, for any measurable test function $\varphi : \mathbb{S}^{d-1} \rightarrow \mathbb{R}$,
\begin{align*}
    \int \varphi \straightd\omega_d &= |\mathbb{S}^{d-1}| \int \varphi d\Tilde{\omega}_d \\
    &= |\mathbb{S}^{d-1}| \int_x \varphi \left( \frac{x}{||x||} \right) \straightd\rho_d(x) \\
    &= |\mathbb{S}^{d-1}| \int_{x_H, x_\perp} \varphi \left( \frac{x_H + x_\perp}{||x_H + x_\perp||} \right) \straightd\rho_{d_H}(x_H) \straightd\rho_{d_\perp}(x_\perp) \\
    &= C_d \int \varphi \left( \frac{r_H z_H + r_\perp z_\perp}{||r_H z_H + r_\perp z_\perp||} \right)  r_H^{d_H-1} e^{-r_H^2/2} r_\perp^{d_\perp-1} e^{-r_\perp^2/2} \straightd r_H \straightd r_\perp \straightd\omega_{d_H}(z_H) \straightd\omega_{d_\perp}(z_\perp) \\
    &= C_d \int_{z_H, z_\perp} \int_{r_H, r_\perp} \varphi \left( \frac{r_H z_H + r_\perp z_\perp}{\sqrt{r_H^2 + r_\perp^2}} \right)  r_H^{d_H-1} r_\perp^{d_\perp-1} e^{-(r_H^2 + r_\perp^2)/2}  \straightd r_H \straightd r_\perp \straightd\omega_{d_H}(z_H) \straightd\omega_{d_\perp}(z_\perp),
\end{align*}
with 
\begin{align*}
    C_d := \frac{|\mathbb{S}^{d-1}|}{(2 \pi)^{d_H/2} (2 \pi)^{d_\perp/2}} = \frac{|\mathbb{S}^{d-1}|}{(2 \pi)^{d/2}} = \frac{2 \pi^{d/2}}{2^{d/2} \pi^{d/2} \Gamma(d/2)} = \frac{1}{2^{(d-2)/2} \Gamma(d/2)}.
\end{align*}
Doing the polar change of variables $(r_H, r_\perp) \in \mathbb{R}_{+}^2 \rightarrow (R, \theta) \in \mathbb{R}_{+} \times [0, \pi/2]$, we get:
\begin{align*}
    \int \varphi \straightd\omega_d &= C'_d \int_{z_H, z_\perp} \int_{ \theta} \varphi \left( \cos(\theta) z_H + \sin(\theta) z_\perp \right)  \cos(\theta)^{d_H-1} \sin(\theta)^{d_\perp-1} d\theta \straightd\omega_{d_H}(z_H) \straightd\omega_{d_\perp}(z_\perp)
\end{align*}
where 
\begin{align*}
    C'_d :=& C_d \int_{0}^{+\infty} R^{d-2} e^{-R^2/2} R dR \\
    =& C_d \int_{0}^{+\infty} R^{d-1} e^{-R^2/2} dR \\
    =& C_d  \times 2^{(d-2)/2} \Gamma(d/2) \\
    =& 1.
\end{align*}
which concludes the proof.
\end{proof}
\begin{remark}
A similar disintegration result holds for the uniform measure $\tilde{\omega}_d$ on the sphere. The corresponding measures which are then pushed-forward by the same $\Phi$ are the normalized counterparts of the measures in the theorem above: $\tilde{\omega}_d = \Phi_{\#} (\tilde{\omega}_{d_H} \otimes \tilde{\omega}_{d_\perp} \otimes \tilde{\gamma})$. This readily comes from noting that a simple calculation yields $|\omega_d| = |\omega_{d_H}| \, |\omega_{d_\perp}|\, |\gamma|$.
\end{remark}

\section{Gradient flows on the space of probability measures}\label{app:measures-gf}

\subsection{First variation of a functional over measures}\label{app:first-variation}

Given a functional $F: \mathcal{P}_2(\mathbb{R}^p) \to \mathbb{R}$, its \textit{first variation} or \textit{Fréchet derivative} at $\mu \in \mathcal{P}_2(\mathbb{R}^p)$ is defined as a measurable function, denoted $\frac{\delta F }{\delta \mu}(\mu) : \mathbb{R}^p \to \mathbb{R}$, such that, for any $\nu \in \mathcal{P}_2(\mathbb{R}^p)$ for which $\mu + t \nu \in \mathcal{P}_2(\mathbb{R}^p)$ in a neighborhood (in $t$) of $t=0$, 
\begin{align*}
    \frac{d }{dt} F(\mu + t \nu) \Big|_{t=0} = \int_{z} \frac{\delta F }{\delta \mu}(\mu)[z] \straightd \nu (z).
\end{align*}
See Santambrogio~\cite[Definition 7.12]{santambrogio2015optimal}, or~\cite[p.29]{santambrogio2017euclidean} for more details on the first variation.

In the case of the functional defined in Equation~\eqref{eq:objective} corresponding to the population loss objective, using the differentiability of the loss $\ell$ \wrt~its second argument, one readily has that
\begin{align*}
    F^\prime_{\mu}(c) :=  \frac{\delta F }{\delta \mu}(\mu)[c] = \int_{x} \partial_2 \ell(f^*(x), f(\mu;x)) \phi(c; x) \straightd \rho(x)
\end{align*}
since
\begin{align*}
    \frac{d }{dt} \ell \Big(f^*(x), f(\mu;x) + t f(\nu; x) \Big) = \partial_2 \ell \Big(f^*(x), f(\mu;x) + t f(\nu; x) \Big) \int_{c} \phi(c; x) \straightd \nu(c).
\end{align*}

\subsection[]{Wasserstein gradient flows in the space $\mathcal{P}_2(\mathbb{R}^{d+1})$}\label{app:w-gf}

A Wasserstein gradient flow for the objective $F$ defined in Equation~\eqref{eq:objective} is a path $(\mu_t)_{t \geq 0}$ in the space of probability measures $\mathcal{P}_2(\mathbb{R}^{d+1})$ which satisfies the continuity equation with a vector field $v_t$ which is equal to the opposite of the gradient of the first variation of the functional $F$. This means that we have, in the sense of distributions,
\begin{align*}
    \partial_t \mu_t = - \text{div} \left(- \nabla \left(\frac{\delta F}{\delta \mu} \right) \mu_t \right).
\end{align*}
That a pair $((\mu_t)_{t \geq 0}, v_t)$ consisting of a path in $\mathcal{P}_2(\mathbb{R}^p)$ and a (time-dependent) vector field in $\mathbb{R}^p$ satisfies the continuity equation $\partial_t \mu_t = - \text{div} (v_t \mu_t)$ in the sense of the distributions simply means that for any test function $\varphi \in \mathcal{C}^1_c(\mathbb{R}^{p})$,
\begin{align*}
    \partial_t \int \varphi \, \straightd \mu_t = \int v_t^\top \nabla \varphi \, \straightd \mu_t,
\end{align*}
where $\partial_t$ stands for the time derivative $\frac{d}{dt}$. Similarly, when we say that the advection-reaction equation $\partial_t \mu_t = - \text{div} \left(v_t \mu_t \right) + g_t \mu_t$ is satisfied for some function $g_t: \mathbb{R}^{p} \to \mathbb{R}$, we mean that it is in the sense of distributions: for any test function $\varphi \in \mathcal{C}^1_c(\mathbb{R}^{p})$,
\begin{align*}
    \partial_t \int \varphi \, \straightd \mu_t = \int (v_t^\top \nabla \varphi + g_t) \, \straightd \mu_t .
\end{align*}
An alternative description of the Wasserstein gradient flow of the objective $F$ is to consider a flow $X_{\bullet}(\cdot)$ in $\mathbb{R}_+ \times \mathbb{R}^{d+1}$ such that, for any $c \in \mathbb{R}^{d+1}$, 
\begin{align*}
    X_0(c) &= c \\
    \frac{d}{dt} X_t(c) &= - \nabla \left(\frac{\delta F}{\delta \mu} \right) \left(X_t(c) \right)
\end{align*}
and to define $\mu_t = (X_t)_{\#} \mu_0$.

For more details on Wasserstein gradient flows in the space of probability measures see Santambrogio~\cite[Section 5.3]{santambrogio2015optimal}, and~\cite[Section 4]{santambrogio2017euclidean}, and for more details on the equivalence between the continuity equation and the flow-based representation of the solution see Santambrogio~\cite[Theorem 4.4]{santambrogio2015optimal}.

\section{Proofs of the symmetry results of Section~\ref{sec:sym}}\label{app:sym}

There are two main ideas behind the proof. Call $\Tilde{T}: (a, b) \in \mathbb{R} \times \mathbb{R}^d \mapsto (\pm a, T(b))$ (depending on whether $f^*$ is invariant or anti-invariant under $T$) and consider the following two facts:

\paragraph{Structure of $\phi((a,b);x)$.}
Since $T$ is orthogonal, so is $\Tilde{T}$, and the structure of $\phi((a,b);x) = a\sigma(b^\top x)$ is such that $\phi(\Tilde{T}(a,b); x) = \pm \phi((a,b); T^{-1}(x))$ because $T$ is orthogonal (its adjoint is thus its inverse).

\paragraph{Conjugate gradients.}
Computing the gradient of a function whose input has been transformed by $\Tilde{T}^{-1}$ is the same as the conjugate action of $\Tilde{T}$ on the gradient: $\nabla (\varphi \circ \Tilde{T}^{-1}) = \Tilde{T} \circ (\nabla \varphi) \circ \Tilde{T}^{-1}$ (this is due to the fact that the adjoint of $\Tilde{T}^{-1}$ is $\tildeT$ because $\Tilde{T}$ is orthogonal). Note that we similarly get $\nabla (\varphi \circ \Tilde{T}) = \Tilde{T}^{-1} \circ (\nabla \varphi) \circ \Tilde{T}$.

\subsection{Preliminaries}\label{app:sym-prelim}

We present here arguments that are present in both the proofs of Proposition~\ref{th:learning-inv} and~\ref{th:learning-anti-inv}. Let $T$ be a linear orthogonal map such that $f^*(T(x)) = \pm f^*(x)$, where the $\pm$ is because we deal with both cases at the same time since the logic is the same. Let $t \geq 0$, and define $\nu_t := \tildeT^{-1}_\# \mu_t$. We aim to show that $(\nu_t)_{t \geq 0}$ is also a Wasserstein gradient flow for the same objective as $(\mu_t)_{t \geq 0}$.

\paragraph{Prediction function.}
Let $x \in \mathbb{R}^d$. We have, using the fact that $T$ is orthogonal (and thus that $\innerprod{T(x)}{y} = \innerprod{x}{T^{-1}(y)}$),
\begin{align*}
    f(\nu_t; x) &= \int_{a,b} a \sigma(b^\top x) \straightd \nu_t(a,b) \\
    &= \int_{a,b} \pm a \sigma(T^{-1}(b)^\top x) \straightd \mu_t(a,b) \\
    &= \pm \int_{a,b} a \sigma(b^\top T(x)) \straightd \mu_t(a,b) \\
    &= \pm f(\mu_t; T(x)).
\end{align*}

\paragraph{Time derivative.}
Let $\varphi \in \mathcal{C}^1_c(\mathbb{R}^d)$. Because $\mu_t$ satisfies the continuity Equation~\eqref{eq:w-gf} in the sense of distributions, and using the remark above on conjugate gradients as well as the orthogonality of $\tildeT$, we have:
\begin{align*}
    \partial_t \int \varphi \straightd \nu_t &= \partial_t \int \varphi \circ \tildeT^{-1} \straightd \mu_t \\
    &= \int \innerprod{\nabla(\varphi \circ \tildeT^{-1})}{v_t}  \straightd \mu_t \\
    &= \int \innerprod{\tildeT \circ \nabla \varphi \circ \tildeT^{-1}}{v_t} \straightd \mu_t \\
    &= \int \innerprod{\nabla \varphi \circ \tildeT^{-1}}{\tildeT^{-1} \circ v_t} \straightd \mu_t \\
    &= \int \innerprod{\nabla \varphi}{\tildeT^{-1} \circ v_t \circ \tildeT} \straightd \nu_t.
\end{align*}

\paragraph{Conjugate velocity field.} The equality above actually shows that $\nu_t$ satisfies the continuity equation with the conjugate velocity field $\tildeT^{-1} \circ v_t \circ \tildeT$ instead of $v_t$. We show below that the former is closely related to the latter (and is in fact equal to $-\nabla F^{\prime}_{\nu_t}$ with sufficient assumptions on $\partial_2 \ell$, which is the step proven in Appendices~\ref{app:proof-learning-inv} and~\ref{app:proof-learning-anti-inv}). Indeed, because $v_t$ is a gradient: $v_t = - \nabla F^{\prime}_{\mu_t}$, we have using again the remark above on conjugate gradients: 
\begin{align*}
    \tildeT^{-1} \circ v_t \circ \tildeT &= - \nabla \left(F^{\prime}_{\mu_t} \circ \tildeT \right).
\end{align*}
Computing the function on the right-hand-side, for any $(a,b) \in \mathbb{R} \times \mathbb{R}^d$, we get, using the remark above on the structure of $\phi$,
\begin{align*}
    F^{\prime}_{\mu_t}( \tildeT(a,b)) &= \int_y \partial_2 \ell \Big(f^*(y), f(\mu_t; y) \Big) \phi \Big(\tildeT(a,b); y \Big) \straightd \rho(y) \\
    &= \pm \int_y \partial_2 \ell \Big(f^*(y), f(\mu_t; y) \Big) \phi \Big((a,b); T^{-1}(y) \Big) \straightd \rho(y).
\end{align*}
$\rho$ is invariant under $T$ since it spherically symmetric by assumption (and thus invariant under any orthogonal map) and we can therefore replace $y$ by $T(y)$ in the integral above, which yields
\begin{align*}
    F^{\prime}_{\mu_t}( \tildeT(a,b)) &= \pm \int_y \partial_2 \ell \Big(f^*(T(y)), f(\mu_t; T(y)) \Big) \phi \Big((a,b); y \Big) \straightd \rho(y) \\
    &= \pm \int_y \partial_2 \ell \Big(\pm f^*(y), \pm f(\nu_t; y) \Big) \phi \Big((a,b); y \Big) \straightd \rho(y),
\end{align*}
and thus we get
\begin{align*}
    \nabla \left(F^{\prime}_{\mu_t} \circ \tildeT \right) (a,b) &= \pm \int_y \partial_2 \ell \Big(\pm f^*(y), \pm f(\nu_t; y) \Big) \nabla_{(a,b)} \phi \Big((a,b); y \Big) \straightd \rho(y).
\end{align*}
One can already notice that if $f^*$ is invariant under $T$ (as opposed to anti-invariant), that is if we keep the \quoting{$+$} in $\pm$, we get $\tildeT^{-1} \circ v_t \circ \tildeT = -\nabla F^{\prime}_{\nu_t}$. 

\subsection{Proof of Proposition~\ref{th:learning-inv}}\label{app:proof-learning-inv}

\begin{proof}
We first prove $\nu_0 = \mu_0$ and then prove that both $(\mu_t)_{t \geq 0}$ and $(\nu_t)_{t \geq 0}$ are Wasserstein gradient flows of the objective $F$ defined in Equation~\eqref{eq:objective}, starting from the initial condition $\mu_0$ at $t=0$. The unicity of such a gradient flow then guarantees that $\mu_t = \nu_t$ and thus $f(\mu_t; T(x)) = f(\mu_t; x)$ by the preliminaries above on the prediction function (see Appendix~\ref{app:sym-prelim}). 

\paragraph{Initialization: $\nu_0 = \mu_0$.}
By definition, $\tildeT(a,b) = (a, T(b))$. Since $\mu_0 = \mu_0^1 \otimes \mu_0^2$ by assumption, and $\mu_0^2$ is invariant under $T$ since it has spherical symmetry, it is clear that $\mu_0$ is invariant under $\tildeT$, and thus under $\tildeT^{-1}$ by Lemma~\ref{th:inv-invertible}, which gives $\nu_0 = \mu_0$ because $\nu_t = \tildeT^{-1}_\# \mu_t$ for any $t$ by definition.
\end{proof}

\paragraph{Time derivative.}
From the preliminary results above (see Appendix~\ref{app:sym-prelim}) we have 
\begin{align*}
    \partial_t \nu_t &= - \text{div} \left(-\nabla F^\prime_{\nu_t} \, \nu_t\right),
\end{align*}
which shows that $(\nu_t)_{t \geq 0}$ is also a Wasserstein gradient flow of the objective $F$. By unicity of the latter (starting from the initial condition $\mu_0$), it must hold that $\mu_t = \nu_t$ for any $t \geq 0$ which concludes the proof. é

\subsection{Proof of Proposition~\ref{th:learning-anti-inv}}\label{app:proof-learning-anti-inv}

The proof follows the exact same pattern as that of Proposition~\ref{th:learning-inv} (see Appendix~\ref{app:proof-learning-inv}). We now have by definition, $\tildeT(a,b) = (-a, T(b))$ and the added symmetry assumption on $\mu_0^1$ ensures that $\nu_0 = \mu_0$ still holds in this case. As for the time derivative, the preliminaries above (see Appendix~\ref{app:sym-prelim}) ensure that
\begin{align*}
    \nabla \left(F^{\prime}_{\mu_t} \circ \tildeT \right) (a,b) &= - \int_y \partial_2 \ell \Big(- f^*(y), - f(\nu_t; y) \Big) \nabla_{(a,b)} \phi \Big((a,b); y \Big) \straightd \rho(y) \\
    &= \int_y \partial_2 \ell \Big(f^*(y), f(\nu_t; y) \Big) \nabla_{(a,b)} \phi \Big((a,b); y \Big) \straightd \rho(y),
\end{align*}
where we have used the extra assumption that $\partial_2 \ell(-y, -\hat{y}) = - \partial_2 \ell(y, \hat{y})$. This yields 
\begin{align*}
    \partial_t \nu_t &= - \text{div} \left(-\nabla F^\prime_{\nu_t} \, \nu_t\right)
\end{align*}
and the conclusion follows from the same logic as for Proposition~\ref{th:learning-inv}.

\section{Proof of the exponential convergence for linear networks: \autoref{th:lin-exp-cv}}\label{app:lin-exp-cv}
\begin{proof}
The proof is divided in three steps: $(i)$ we derive the dynamics in time of the vector $w(t) = \frac{1}{2} \int a b \, \straightd \mu_t(a,b)$, $(ii)$ we show that the positive definite matrix $H(t)$ appearing in these dynamics has its smallest eigenvalue lower-bounded by some positive constant after some $t_0 > 0$, and $(iii)$ we show that this implies the exponential convergence to the global minimum.

\paragraph{Generalities on the objective $Q$.}
Expanding the square in the definition of $Q$~\eqref{eq:Q}, we have
\begin{align*}
    Q(w) &= \frac{1}{2} \Big[\mathbb{E}_{x \sim \mathbb{P}}[f^*(x)^2] - 2 \beta^\top w + w^\top C w \Big],\\
    C :&= \mathbb{E}_{x \sim \mathbb{P}}[x x^\top] \in \mathbb{R}^{d \times d},\\
    \beta :&= \mathbb{E}_{x \sim \mathbb{P}}[f^*(x) x] \in \mathbb{R}^{d}.
\end{align*}
If $C \neq 0$, $Q(w) \to \infty$ as $||w|| \to \infty$ and since $Q$ is lower-bounded by $0$, it thus admits at least one global minimum. This minimizer $w^\star$ is unique as soon as $Q$ is strongly convex, \ie, $C$ is definite positive, which holds in this case as we have assumed the smallest eigenvalue $\lambdamin$ of $C$ to be $>0$. Note that $\nabla Q(w) = Cw - \beta = \int_x \left((x^\top w) - f^*(x) \right) x \straightd \mathbb{P}(x) \in \mathbb{R}^d$.

\paragraph{First step: dynamics of $w(t)$.}
Let $k \in \{1, \ldots, d\}$, the $k$-th coordinate $w_k(t)$ of $w(t)$ is given by $w_k(t) = \int a b_k \, \straightd \mu_t(a,b)$, and its time derivative is given by 
\begin{align*}
    w^\prime_k(t) = \frac{1}{2} \int \Big(\nabla_{(a,b)} (a b_k) \Big)^\top v_t(a,b) \straightd \mu_t(a,b)
\end{align*}
where $v_t$ is given by Equation~\eqref{eq:w-gf} except we replace $\sigma$ by $\frac{1}{2} \text{id}_{\mathbb{R}^d}$ and $\rho$ by $\mathbb{P}$ in $F_{\mu_t}$, that is
\begin{align*}
    v_t(a,b) = \frac{1}{2} \int_y R_t(y) \begin{pmatrix}
        b^\top y \\
        a y
    \end{pmatrix} \straightd \mathbb{P}(y) \in \mathbb{R}^{1+d}.
\end{align*}
On the other hand, $\nabla_{(a,b)} (a b_k) = \begin{pmatrix}
        b_k \\
        a e_k
\end{pmatrix} \in \mathbb{R}^{1 + d}$ where $e_k$ is the $k$-th element of the canonical orthonormal basis of $\mathbb{R}^d$. Note that here, $R_t(y) = f^*(y) - \langle w(t), y \rangle$.
We thus get
\begin{align*}
    w^\prime_k(t) =& \ \frac{1}{4}  \left \langle \int_{a,b} b_k b \straightd \mu_t(a,b),  \int_y \left(f^*(y) - (w(t)^\top y) \right) y \straightd \mathbb{P}(y) \right \rangle + \ \\ & \ \frac{1}{4}  \left \langle \int_{a,b} a^2 e_k \straightd \mu_t(a,b),  \int_y \left(f^*(y) - (w(t)^\top y) \right) y \straightd \mathbb{P}(y) \right \rangle.
\end{align*}
Note that the term on the right in the inner products is in fact equal to $-\nabla Q(w(t))$, which yields the following dynamics for the vector $w(t)$:
\begin{align*}
    w^\prime(t) &= - H(t) \nabla Q(w(t)), \\
    H(t) :&= \frac{1}{4} \left(\int b b^\top \straightd \mu_t(a,b) + \int a^2 \straightd \mu_t(a,b) I_d \right) \in \mathbb{R}^{d \times d}.
\end{align*}

\paragraph{Second step: lower bound on the smallest eigenvalue of $H(t)$.}

At initialization, by symmetry one has $w(0) = 0$, and using Lemma~\ref{th:spherical-sym-id}, one has that $H(0) = \frac{1}{4} \left( \frac{1}{d} + 1\right) I_d$, so that
\begin{align*}
    \frac{d}{dt} Q(w(t)) \Big|_{t=0} &= \left \langle w^\prime(0), \nabla Q(w(0)) \right \rangle \\
    &= -\frac{d+1}{4d} ||\nabla Q(0)||^2 \\
    &= -\frac{d+1}{4d} ||\beta||^2
\end{align*}
If $\beta = 0$, then $\nabla Q(0) = 0$ and since $w(0)=0$, $w(t)$ starts at the global optimum and thus stays constant equal to $0$. Otherwise, if $||\beta|| > 0$, one has $\frac{d}{dt} Q(w(t)) \Big|_{t=0} < 0$, which ensures that there is a $t_0 > 0$ such that $Q(w(t)) < Q(w(0)) = Q(0)$ for any $t \in (0, t_0]$. Call $\varepsilon := \left[Q(0) - Q(w(t_0)) \right]/2 > 0$. The continuity of $Q$ at $0$ guarantees that there is a $\delta > 0$ such that for any $w \in \mathbb{R}^d$, if $||w|| < \delta$, then $|Q(w) - Q(0)| \leq \varepsilon$. 

Now assume that there exists $t_1 \geq t_0$ such that $\int a^2 \straightd \mu_{t_1}(a,b) \leq \delta$. Then, one has
\begin{align*}
    ||w(t_1)|| &= \left| \left| \frac{1}{2} \int a b \straightd \mu_{t_1}(a,b)  \right| \right| \\
    &\leq \frac{1}{2} \int |a| \, ||b|| \straightd \mu_{t_1}(a,b) \\
    &\leq \frac{1}{2} \int a^2 \straightd \mu_{t_1}(a,b)\\
    &\leq \frac{\delta}{2} < \delta,
\end{align*}
where we have used in the penultimate inequality that $\mu_{t_1}$ is supported on the set $\{ |a| = ||b|| \}$ because of the assumptions on the initialization $\mu_0$ (see Section~\ref{sec:pb-setting}). This ensures that $|Q(w(t_1)) - Q(0)| \leq \varepsilon$. Since $:t \mapsto Q(w(t))$ is decreasing ($Q(w(t)) = F(\mu_t)$ and it is classical that the objective is decreasing along the gradient flow path, see third step below) and $t_1 \geq t_0$, this means that
\begin{align*}
     0 < Q(0) - Q(w(t_0)) \leq Q(0) - Q(w(t_1)) \leq \varepsilon = \left[Q(0) - Q(w(t_0)) \right]/2
\end{align*}
which is a contradiction. Therefore, for any $t \geq t_0$, $\int a^2 \straightd \mu_t(a,b) \geq \delta$. Calling $\eta := \delta/4 > 0$, we thus have that for any $t \geq t_0$, the smallest eigenvalue of $H(t)$ is larger than $\eta$ because $H(t)$ is the sum of the positive semi-definite matrix $\frac{1}{4} \int b b^\top \straightd \mu_t(a,b)$ and of the positive definite matrix $\frac{1}{4} \int a^2 \straightd \mu_t(a,b) I_d$ whose smallest eigenvalue is at least $\eta$ for $t \geq t_0$.

\paragraph{Third step: exponential convergence.}
We have:
\begin{align*}
    \frac{d}{dt} Q(w(t)) &= \langle w^\prime(t), \nabla Q(w(t)) \\
    &= - \nabla Q(w(t))^\top H(t) \nabla Q(w(t)) \leq 0,
\end{align*}
which shows that because $H(t)$ is positive definite, the objective $Q$ is decreasing along the path $(w(t))_{t \geq 0}$.
Since after $t_0 > 0$, the smallest eigenvalue of $H(t)$ is lower bounded by a constant $\eta > 0$, we have that, for any $t \geq t_0$:
\begin{align}\label{eq:der-Q}
    \frac{d}{dt} Q(w(t)) &\leq - \eta ||\nabla Q(w(t))||^2.
\end{align}
Because $Q$ is $\lambdamin$-strongly convex (as the smallest eigenvalue of $C$ is $\lambdamin>0$), one has the classical inequality
\begin{align*}
    \frac{1}{2} ||\nabla Q(w)||^2 \geq \lambdamin \Big(Q(w) -Q(w^\star)\Big).
\end{align*}
Plugging this into Equation~\eqref{eq:der-Q} gives
\begin{align*}
    \frac{d}{dt} \Big(Q(w(t)) - Q(w^\star) \Big) &\leq - 2 \eta \lambdamin \Big(Q(w(t)) - Q(w^\star) \Big),
\end{align*}
which by Gronwall's lemma in turn yields for any $t \geq t_0$
\begin{align*}
    0 \leq Q(w(t)) - Q(w^\star) \leq e^{-2 \eta \lambdamin (t-t_0)} \Big(Q(w(t_0)) - Q(w^\star) \Big),
\end{align*}
thereby proving exponential convergence. 

\paragraph{Exponential convergence in distance.}
Given that $\nabla Q(\wstar) = 0$ because $\wstar$ is and optimum, it holds $C \wstar = \beta$. Using this fact, it easily follows that 
\begin{align*}
     Q(w) - Q(w^\star) = \frac{1}{2} \langle C (w - \wstar), w - \wstar \rangle,
\end{align*}
and the right-hand-side is lower bounded by $\frac{\lambdamin}{2} ||w - \wstar||^2$,
from which we conclude that
\begin{align*}
    ||w(t) - w^\star||^2 \leq \frac{2}{\lambdamin} \Big(Q(w(t)) - Q(w^\star) \Big),
\end{align*}
and the exponential decrease of the right-hand-side allows to conclude.

\end{proof}

\section[]{Proofs of Section~\ref{sec:low-dim}: $f^*$ depends only on the projection on a sub-space $H$}\label{app:low-dim}

\subsection{The general case}\label{app:low-dim-gen}

\subsubsection[]{Closed dynamics on the sphere $\mathbb{S}^{d-1}$}\label{app:low-dim-nu}
We wish to show here that the pair of measures $(\nu_t^+, \nu_t^-)$ defined through Equation~\eqref{eq:nu-pm} satisfy Equation~\eqref{eq:nu_t} and that the corresponding dynamic is closed is the sense that it can be expressed solely using $(\nu_t^+, \nu_t^-)$ (without requiring to express quantities in function of $\mu_t$). Below, we use $\kappa(z) = \max(0, z)$. We do this do differentiate it from the activation function $\sigma$ (which is also equal to ReLU) so as avoid confusion because the $\kappa$ which appears below has nothing to do with the activation function of the network and simply comes from the integration domain in the calculations.

\paragraph{Equations of the dynamics on the sphere.}
Let $\varphi \in \mathcal{C}^1_c(\mathbb{S}^{d-1})$. One has
\begin{align*}
    \partial_t \int \varphi \straightd \nu_t^\pm =& \ \partial_t \int_{\pm a \geq 0, b} |a| ||b|| \varphi \left( \frac{b}{||b||} \right) \straightd \mu_t(a,b) \\
    =& \ \partial \int_{a,b} \kappa(\pm a) ||b|| \varphi \left( \frac{b}{||b||} \right) \straightd \mu_t(a,b) \\
    =& \ \int_{a,b} \nabla_{(a,b)}  \left( \kappa(\pm a) ||b|| \varphi  \left( \frac{b}{||b||} \right) \right)^\top v_t(a,b) \straightd \mu_t(a,b)
\end{align*}
Let us compute the components of the gradient above. We have 
\begin{align*}
    \nabla_{a}  \left( \kappa(\pm a) ||b|| \varphi  \left( \frac{b}{||b||} \right) \right) = \pm \kappa^\prime(\pm a) ||b|| \varphi  \left( \frac{b}{||b||} \right) = \mathbf{1}_{\{ \pm a \geq 0 \}} ||b|| \varphi  \left( \frac{b}{||b||} \right).
\end{align*}
The Jacobian of the map $:b \in \mathbb{R}^d \mapsto b / ||b||$ is equal to $\frac{1}{||b||}(I_d - b b^\top / ||b||^2)$ which is a symmetric (or self-adjoint) matrix, so that the gradient \wrt~$b$ is 
\begin{align*}
    \nabla_{b}  \left( \kappa(\pm a) ||b|| \varphi  \left( \frac{b}{||b||} \right) \right) = \mathbf{1}_{\{ \pm a \geq 0 \}} |a| \left[ \varphi \left( \frac{b}{||b||} \right) \frac{b}{||b||} + \left(I_d - \frac{b}{||b||} \left(\frac{b}{||b||} \right)^\top \right) \nabla \varphi \left( \frac{b}{||b||} \right) \right].
\end{align*}
On the other hand, the first component of $v_t(a,b)$ (corresponding to the gradient \wrt~$a$) is
\begin{align*}
    v_t^1(a,b) = \int_y R_t(y) \kappa(b^\top y) \straightd \rho(y) = ||b|| \int_y R_t(y) \kappa \left ( \left(\frac{b}{||b||} \right)^\top y \right) \straightd \rho(y),
\end{align*}
and the last $d$ components (corresponding to the gradient \wrt~$b$) are
\begin{align*}
    v_t^2(a,b) = \int_y R_t(y) a \kappa^\prime(b^\top y) y \straightd \rho(y) = a  \int_y R_t(y) \kappa^\prime \left ( \left(\frac{b}{||b||} \right)^\top y \right) y \straightd \rho(y).
\end{align*}
When computing the inner product $\nabla_{(a,b)} \left( \kappa(\pm a) ||b|| \varphi  \left( \frac{b}{||b||} \right) \right)^\top v_t(a,b)$, we can re-arrange the terms to keep one term where $\varphi$ appears and the other where $\nabla \varphi$ appears. Using the facts that the Jacobian computed above is symmetric, that $\kappa(z) = \kappa^\prime(z) z$ for any $z \in \mathbb{R}$, and that $\mathbf{1}_{\{ \pm a \geq 0 \}} a = \pm \mathbf{1}_{\{ \pm a \geq 0 \}} |a| = \pm \kappa(\pm a)$, we get,
\begin{align*}
    \nabla_{(a,b)} \left( \kappa(\pm a) ||b|| \varphi  \left( \frac{b}{||b||} \right) \right)^\top v_t(a,b) =& \ \pm  \mathbf{1}_{\{\pm a \geq 0 \}} ||b|| \, ||b|| \varphi \left(\frac{b}{||b||} \right) g_t \left(\frac{b}{||b||} \right) + \ \\
    & \ \pm  \mathbf{1}_{\{\pm a \geq 0 \}} |a| \, |a| \varphi \left(\frac{b}{||b||} \right) g_t \left(\frac{b}{||b||} \right) + \ \\
    & \ \pm \mathbf{1}_{\{\pm a \geq 0 \}} |a| \, |a| \nabla \varphi \left(\frac{b}{||b||} \right)^\top \Tilde{v}_t \left(\frac{b}{||b||} \right),
\end{align*}
where, for $u \in \mathbb{S}^{d-1}$
\begin{align*}
    g_t(u) &:= \int_y R_t(y) \sigma(u^\top y) \straightd\rho(y), \\
    \Tilde{v}_t(u) &:= \int_y R_t(y) \sigma^\prime(u^\top y) \left[y - (u^\top y) u\right] \straightd\rho(y).
\end{align*}
Finally, because $\mu_t$ stays on the cone $\{(a,b) \in \mathbb{R}^{d+1} ; |a| \ = ||b||\}$ for any $t$ (see Chizat and Bach~\cite[Lemma 26]{chizat2020implicitBias}, Wojtowytsch~\cite[Section 2.5]{wojtowytsch2020convergence}), when integrating against $\mu_t$, we can replace $||b||$ by $|a|$ and vice-versa. We thus get that the time derivative we initially computed is the sum of two terms:
\begin{align*}
    \partial_t \int \varphi \straightd \nu_t^\pm =& \ 2 \int_{\pm a \geq 0, b} |a| ||b|| \varphi \left(\frac{b}{||b||} \right) g_t \left(\frac{b}{||b||} \right) \straightd \mu_t(a,b) \ + \\
    & \ \int_{\pm a \geq 0, b} |a| ||b|| \nabla \varphi \left(\frac{b}{||b||} \right)^\top \Tilde{v}_t \left(\frac{b}{||b||} \right) \straightd \mu_t(a,b) \\
    =& \ 2 \int_{u \in \mathbb{S}^{d-1}}  \varphi \left(u \right) g_t \left(u \right) \straightd \nu^{\pm}_t(u) \ + \\
    & \ \int_{u \in \mathbb{S}^{d-1}}  \nabla \varphi \left(u \right)^\top \Tilde{v}_t \left(u \right) \straightd \nu^{\pm}_t(u),
\end{align*}
which shows that $\nu^\pm_t$ satisfies Equation~\eqref{eq:nu_t} in the sense of distributions.

\paragraph{Closed dynamics.} We want to show that $g_t$ and $\tilde{v}_t$ can be expressed using only $\nu_t^+$ and $\nu_t^-$. Both these quantities depend on $t$ only through the residual $R_t$, which itself only depends on $t$ through $f(\mu_t; \cdot)$. We thus show that the latter can be expressed using only $\nu_t^+$ and $\nu_t^-$, which easily follows from writing, for any $y \in \mathbb{R}^d$, 
\begin{align*}
    f(\mu_t; y) =& \int a \sigma \left(b^\top y \right) \straightd \mu_t(a,b) \\
    =& \int a ||b|| \sigma \left(\left \langle \frac{b}{||b||}, \, y \right\rangle \right) \straightd \mu_t(a,b) \\
    =& \int_{a \geq 0, b} |a| \, ||b|| \sigma \left(\left \langle \frac{b}{||b||}, \, y \right\rangle \right) \straightd \mu_t(a,b) - \int_{a \leq 0, b} |a| \, ||b|| \sigma \left(\left \langle \frac{b}{||b||}, \, y \right\rangle \right) \straightd \mu_t(a,b) \\
    =& \int_{u \in \mathbb{S}^{d-1}} \sigma \left(u^\top y \right) \straightd \nu_t^+(u) - \int_{u \in \mathbb{S}^{d-1}} \sigma \left(u^\top y \right) \straightd \nu_t^-(u)
\end{align*}

\subsubsection[]{Closed dynamics in $d_H + 1$ dimensions}\label{app:low-dim-tau}

\begin{proof}
We first prove that the Equation~\eqref{eq:tau_t-general} for $\tau_t^\pm$ holds in the sense of distributions, and then show that the corresponding dynamics are \textbf{closed} because the $V_t$ and $g_t$ appearing in Equation~\eqref{eq:tau_t-general} can be expressed with $(\tau_t^+, \tau_t^-)$ (and not only with $(\nu_t^+, \nu_t^-)$ for instance). We show this by expressing $f(\mu_t; \cdot)$ only in function of the pair $(\tau_t^+, \tau_t^-)$.

\paragraph{The pair $(\tau_t^+, \tau_t^-)$ satisfy Equation~\eqref{eq:tau_t-general}.}
First, we show that $g_t$ and $\Tilde{v}_t$ defined in Equation~\eqref{eq:advec-reac-general} admit modified expressions that match the structure of the pushforward transforming $\nu_t^\pm$ into $\tau_t^\pm$. Indeed, since $\rho$ is assumed to be spherically symmetric, it is invariant by any orthogonal transformation. In particular, for a fixed $u \in \mathbb{S}^{d-1}$ such that $u^\perp \neq 0$, we consider the orthogonal map $T^u: \mathbb{R}^d \rightarrow \mathbb{R}^d$ such that $T^u_{|H} = \text{id}_H$ and $T^u_{|H^\perp}$ sends the canonical orthonormal basis $(e^\perp_1, \ldots, e^\perp_{d_\perp})$ of $H^\perp$ on $(u^\perp/||u^\perp||, u_2, \ldots, u_{d_\perp})$ where $(u_2, \ldots, u_{d_\perp}) \in (H^\perp)^{d_\perp - 1}$ is an  orthonormal family, orthogonal to $u^\perp$, so that for any $y^\perp \in H^\perp$ with coordinates $y^\perp_1, \ldots, y^\perp_{d_\perp}$ in the basis $(e^\perp_1, \ldots, e^\perp_{d_\perp})$, $T^u_{|H^\perp}(y^\perp) = y^\perp_1 u^\perp/||u^\perp|| + h_u(y_\perp)$ with $h_u(y_\perp) \perp u^\perp$. 

Note that since $f^*(y) = f_H(y^H)$ and $f(\mu_t; y) = \tilde f_t(y^H, ||y^\perp||)$, the residual $R_t(y) = f^*(y) - f(\mu_t; y)$ is invariant by any orthogonal transformation which preserves $H$ (and in particular by $T^u$). We thus have
\begin{align*}
    g_t(u) =& \int_y R_t(y) \sigma \left(\innerprod{u^H}{y^H} + y^\perp_1 ||u^\perp|| \right) \straightd \rho =: \Tilde{g}_t(u^H, ||u^\perp||), \\
    \Tilde{v}_t(u) =& \int_y R_t(y) \sigma^\prime\left(\innerprod{u^H}{y^H} + y^\perp_1 ||u^\perp|| \right) \left[y^H + T^u_{|H^\perp}(y^\perp) - \left(\innerprod{u^H}{y^H} + y^\perp_1 ||u^\perp|| \right)u \right] \straightd \rho.
\end{align*}
Now consider, for any $(\theta, z^H) \in [0, \pi/2] \times \mathbb{S}^{d_H-1}$,
\begin{align*}
    G_t(\theta, z^H) :&= \Tilde{g}_t(\cos(\theta) z^H, \sin(\theta)) \\
    V_t(\theta, z^H) :&= \int_{y} R_t(y) \sigma^\prime \left(\cos(\theta) \innerprod{z^H}{y^H} + y^\perp_1 \sin(\theta) \right)
    \begin{pmatrix}
        y^\perp_1 \cos(\theta) - \sin(\theta) \innerprod{z^H}{y^H} \\
        \frac{y^H}{\cos(\theta)}   - \innerprod{z^H}{y^H} \frac{z^H}{\cos(\theta)}
    \end{pmatrix}
    \straightd \rho
\end{align*}
We show below that $(\tau_t^+, \tau_t^-)$ satisfy Equation~\eqref{eq:tau_t-general} with the $G_t$ and $V_t$ defined above. Let $\varphi \in \mathcal{C}^1_c([0, \pi/2] \times \mathbb{S}^{d_H-1})$. Since $\tau_t^\pm$ is defined as a push-forward measure obtained from $\nu_t^\pm$ we have:
\begin{align*}
    \partial_t \int \varphi(\theta, z^H) \straightd \tau_t^\pm(\theta, z^H) =& \ \partial_t \int \varphi \left(\arccos(||u^H||), \frac{u^H}{||u^H||} \right) \straightd \nu_t^\pm(u) \\
    =& \pm 2 \int \varphi \left(\arccos(||u^H||), \frac{u^H}{||u^H||} \right) \Tilde{g}_t(u^H, ||u^\perp||) \straightd \nu_t^\pm(u) \  + \\
    &\pm \int \nabla_u \left(\varphi \left(\arccos(||u^H||), \frac{u^H}{||u^H||} \right) \right)^\top \Tilde{v}_t(u) \straightd \nu_t^\pm(u).
\end{align*}
By definition of the pushforward, and since $u^H = \cos(\arccos(||u^H||)) u^H/||u^H||$ and $||u^\perp|| = \sin(\arccos(||u^H||))$ for $u \in \mathbb{S}^{d-1}$, the first integral is equal $\int \varphi(\theta, z^H) G_t(\theta, z^H) \straightd \tau_t^\pm(u)$. For the second integral, let us first compute the gradient. One has
\begin{align*}
    \nabla_u \left(\varphi \left(\arccos(||u^H||), \frac{u^H}{||u^H||} \right) \right) =& \partial_\theta \varphi \left(\arccos(||u^H||), \frac{u^H}{||u^H||} \right) \frac{-1}{\sqrt{1 - ||u^H||^2}} \frac{u^H}{||u^H||} \ + \\ &\frac{1}{||u^H||} \left[I_{d_H} - \frac{u^H {(u^H)}^\top}{||u^H||^2} \right] \left(\nabla_{z^H} \varphi \right) \left(\arccos(||u^H||), \frac{u^H}{||u^H||} \right).
\end{align*}
We observe that the gradient above belongs to $H$ which implies that when computing its inner product with $\tilde{v}_t(u)$ we can consider only the component of the latter along $H$. Additionally, we note that $I_{d_H} - u^H (u^H)^\top / ||u^H||^2$ is actually the orthogonal projection onto $\{u^H\}^\perp$, so that it yields $0$ when applied to $u$. Using that $||u^\perp|| = \sqrt{1 - ||u^H||^2}$ for $u \in \mathbb{S}^{d-1}$, we then get:
\begin{align*}
    \nabla_u \left(\varphi \left(\arccos(||u^H||), \frac{u^H}{||u^H||} \right) \right)^\top \Tilde{v}_t(u) &= \nabla \varphi \left(\arccos(||u^H||), \frac{u^H}{||u^H||} \right)^\top V_t \left((\arccos(||u^H||), \frac{u^H}{||u^H||} \right).
\end{align*}
where $\nabla \varphi (\theta, z^H) = 
\begin{pmatrix}
    \partial_\theta \varphi (\theta, z^H) \\
    \nabla_{z^H} \varphi (\theta, z^H)
\end{pmatrix}$. This shows that 
\begin{align*}
    \partial_t \int \varphi (\theta, z^H) \straightd \tau_t^\pm(\theta, z^H) =& \pm 2 \int \varphi(\theta, z^H) G_t(\theta, z^H) \straightd \tau_t^\pm(\theta, z^H) \ + \\
    &\pm \int \nabla \varphi (\theta, z^H)^\top V_t(\theta, z^H) \straightd \tau_t^\pm(\theta, z^H),
\end{align*}
which proves that $\tau_t^\pm$ indeed satisfies Equation~\eqref{eq:advec-reac-general} in the sense of distributions. 

\paragraph{The dynamics are closed in the pair $(\tau_t^+, \tau_t^-)$.} The only thing left to prove to show that the dynamics are closed for the pair $(\tau_t^+, \tau_t^-)$ is that $G_t$ and $V_t$ can be expressed using only the pair $(\tau_t^+, \tau_t^-)$. The only dependence of these quantities on $t$ is through the residual $R_t$ which itself depends on $t$ only through $f(\mu_t; \cdot)$. Let $y \in \mathbb{R}^d$. We have already shown at the end of the previous Section~\ref{app:low-dim-nu} that by definition of $\nu_t^+$ and $\nu_t^-$, we have
\begin{align*}
    f(\mu_t; y) =& \int_{u \in \mathbb{S}^{d-1}} \sigma \left(u^\top y \right) \straightd \left(\nu_t^+ - \nu_t^-\right)(u).
\end{align*}
On the other hand, we show below that the integral of any measurable function $\varphi: \mathbb{S}^{d-1} \rightarrow \mathbb{R}$ against $\nu_t^\pm$ can be expressed as an integral against $\tau_t^\pm$ in the case where $\nu_t^\pm$ admits a density \wrt~the uniform measure on $\mathbb{S}^{d-1}$ (which is the case for $\nu_0^\pm$), the case of a general measure $\nu_t^\pm$ being a simple extension via a weak convergence argument. Thus call $p_t^\pm$ the density of $\nu_t^\pm$ \wrt~$\tilde{\omega}_d$. Since $\nu_t^\pm$ is invariant by any linear map $T$ such that $T_{|H} = \text{id}_{H}$ $T_{|H^\perp} \in \mathcal{O}(d_\perp)$ (because of the symmetries on $\mu_t$ given by Proposition~\ref{th:learning-inv}), and since this is also the case for $\tilde{\omega}_d$ because $\tilde{\omega}_d$ has spherical symmetry and $T$ is orthogonal, we have by Lemma~\ref{th:inv-density} that $p_t^\pm$ is invariant by any such $T$, which then leads to $p_t$ having the form $p_t(u) = \Tilde{p}_t^\pm(u^H, ||u^\perp||)$ by Lemma~\ref{th:inv-sub-orthogonal}.

\paragraph{First step.}
We show that $\tau_t^\pm$ has the density $q_t^\pm(\theta, z^H) = |\mathbb{S}^{d_\perp-1}| \Tilde{p}_t^\pm(\cos(\theta) z^H, \sin(\theta))$ \wrt~$\tilde{\gamma} \otimes \tilde{\omega}_{d_H}$ where the measure $\tilde{\gamma}$ is the normalized counterpart of the measure in Definition~\ref{def:gamma}. Indeed, let $\varphi: [0, \pi/2] \times \mathbb{S}^{d-1} \rightarrow \mathbb{R}$ be any measurable function \wrt~$\tau_t^\pm$. Using the disintegration Lemma~\ref{th:disintegration} on $\tilde{\omega}_d$, one has that
\begin{align*}
    \int \varphi(\theta, z^H) \straightd \tau_t^\pm(\theta, z^H) =& \int \varphi \left(\arccos(||u^H||), \frac{u^H}{||u^H||} \right) \straightd \nu_t^\pm(u) \\
    =& \int \varphi \left(\arccos(||u^H||), \frac{u^H}{||u^H||} \right) \Tilde{p}_t^\pm(u^H, ||u^\perp||) \straightd \tilde{\omega}_d(u) \\
    =& \int \varphi \left(\theta, z^H \right) \Tilde{p}_t^\pm(\cos(\theta) z^H, \sin(\theta)) \straightd \tilde{\gamma}(\theta) \straightd \tilde{\omega}_{d_H}(z^H) \straightd \tilde{\omega}_{d_\perp}(z^\perp) \\
    =& \int \varphi \left(\theta, z^H \right) \Tilde{p}_t^\pm(\cos(\theta) z^H, \sin(\theta)) \straightd \tilde{\gamma}(\theta) \straightd \tilde{\omega}_{d_H}(z^H),
\end{align*}
which proves the desired density for $\tau_t^\pm$. 
\paragraph{Second step.}
Consider a measurable $\varphi: \mathbb{S}^{d-1} \rightarrow \mathbb{R}$ \wrt~$\nu_t^\pm$. One has with similar calculations as above
\begin{align*}
    \int_{u} \varphi(u) \straightd \nu_t^\pm(u) =& \int_{u} \varphi(u)  \Tilde{p}_t^\pm(u^H, ||u^\perp||) \straightd \tilde{\omega}_d(u) \\
    =& \int  \varphi(\cos(\theta) z^H + \sin(\theta) z^\perp)  \Tilde{p}_t^\pm(\cos(\theta) z^H, \sin(\theta)) \straightd \tilde{\gamma}(\theta) \straightd \tilde{\omega}_{d_H}(z^H) \straightd \tilde{\omega}_{d_\perp}(z^\perp) \\
    =& \int_{\theta, z^H} \left(\int_{z^\perp} \varphi(\cos(\theta) z^H + \sin(\theta) z^\perp) \straightd \tilde{\omega}_{d_\perp}(z^\perp) \right) q_t^\pm(\theta, z^H) \straightd \tilde{\gamma}(\theta) \straightd \tilde{\omega}_{d_H}(z^H) \\
    =& \int_{\theta, z^H} \left(\int_{z^\perp} \varphi(\cos(\theta) z^H + \sin(\theta) z^\perp) \straightd \tilde{\omega}_{d_\perp}(z^\perp) \right)  \straightd \tau_t^\pm(\theta, z^H).
\end{align*}
Applying this to $f(\mu_t; y)$ shows that the latter quantity can be expressed solely using $(\tau_t^+, \tau_t^-)$, which proves that the dynamics is indeed closed and therefore concludes the proof when $\nu_t^\pm$ has a density.

\paragraph{Third step: extending to any measure.}
It is known that for any measure $\nu$ over $\mathbb{S}^{d-1}$, there exists a sequence of measure $(\nu_n)_{n \in \mathbb{N}}$ such that: $(i)$ $\nu_n$ has a density $p_n$ \wrt~the uniform measure $\tilde{\omega}_d$ over $\mathbb{S}^{d-1}$, and $(ii)$ the sequence $(\nu_n)_{n \in \mathbb{N}}$ converges weakly to $\nu$, that is, for any continuous (and thus automatically bounded because the unit sphere is compact) $\varphi$, $\int \varphi \straightd \nu_n \xrightarrow[n \to \infty]{} \int \varphi \straightd \nu$. Let thus $\nu \in \mathcal{M}_+(\mathbb{S}^{d-1})$, and consider a sequence $(\nu_n)_{n \in \mathbb{N}}$ with density converging weakly towards $\nu$. Let $\tau$ (\textit{resp.}~$\tau_n$) be defined from $\nu$ (\textit{resp.}~$\nu_n$) as $\tau_t^\pm$ is defined from $\nu_t^\pm$, that is for any measurable $\varphi : [0, \pi/2] \times \mathbb{S}^{d_H-1} \to \mathbb{R}$,
\begin{align*}
    \int \varphi(\theta, z^H) \straightd \tau(\theta, z^H) =& \int \varphi \left(\arccos(||u^H||), \frac{u^H}{||u^H||} \right) \straightd \nu(u), \\
    \int \varphi(\theta, z^H) \straightd \tau_n(\theta, z^H) =& \int \varphi \left(\arccos(||u^H||), \frac{u^H}{||u^H||} \right) \straightd \nu_n(u).
\end{align*}
Let thus $\varphi$ be a continuous map from $\mathbb{S}^{d-1} \to \mathbb{R}$ (having in mind the example of $:u \mapsto \sigma(u^\top y)$ for a fixed $y$). By the result of Step $2$, since $\nu_n$ has a density for every $n$, we have that 
\begin{align}\label{eq:nu-tau-n}
    \int \varphi(u) \straightd \nu_n(u) =& \int_{\theta, z^H} \left(\int_{z^\perp} \varphi(\cos(\theta) z^H + \sin(\theta) z^\perp) \straightd \tilde{\omega}_{d_\perp}(z^\perp) \right)  \straightd \tau_n(\theta, z^H),
\end{align}
and taking the limit $n \to \infty$, the left-hand-side of Equation~\eqref{eq:nu-tau-n} converges to $\int \varphi \straightd \nu$ by assumption. Now let us look at the right-hand-side of~\eqref{eq:nu-tau-n}. Calling $\psi(\theta, z^H) = \int_{z^\perp} \varphi(\cos(\theta) z^H + \sin(\theta) z^\perp) \straightd \tilde{\omega}_{d_\perp}(z^\perp)$ and $\Phi(u) = \int_{z^\perp} \varphi(u^H + ||u^\perp|| z^\perp) \straightd \tilde{\omega}_{d_\perp}(z^\perp)$, the right-hand-side is in fact $\int \psi \straightd \tau_n$ and, for any $n \in \mathbb{N}$, is equal to:
\begin{align*}
    \int \psi \straightd \tau_n
    &= \int_u \psi \left(\arccos(||u^H||), \frac{u^H}{||u^H||} \right) \straightd \nu_n(u) \\
    &= \int_u \int_{z^\perp} \varphi \left(||u^H|| \frac{u^H}{||u^H||} + ||u^\perp|| z^\perp \right) \straightd \tilde{\omega}_{d_\perp}(z^\perp)  \straightd \nu_n(u) \\
    &= \int_u \int_{z^\perp} \varphi \left(u^H + ||u^\perp|| z^\perp \right) \straightd \tilde{\omega}_{d_\perp}(z^\perp)  \straightd \nu_n(u) \\
    &= \int \Phi \straightd \nu_n,
\end{align*}
and a similar result holds for $\tau$ and $\nu$. Now, the continuity of $\Phi$ is readily obtained from that of $\varphi$, and thus the right-hand-side in the last equality above converges to $\int \Phi \straightd \nu$ which is also equal to $\int \psi \straightd \tau$ by the same calculations as above.  The right-hand-side in~\eqref{eq:nu-tau-n} therefore converges to $\int \psi \straightd \tau$, and since the limits of both sides are equal, we get $\int \varphi \straightd \nu = \int \psi \straightd \tau$, which is the claim of Step 2 for a general measure $\nu$ which does not  necessarily admit a density, thereby concluding the proof.
\end{proof}

\subsection[]{Case when $f^*$ is the euclidean norm:~\autoref{th:1d-theta}}\label{app:1d}

Here, we give the proof of~\autoref{th:1d-theta} which shows that when $f^*(x) = ||x^H||$ the dynamics can be reduced to a single variable: the angle $\theta \in [0, \pi/2]$ between particles and the subs-space $H$. 

We decompose the proof in three steps: first we show that the pair of measures $(\tau_t^+, \tau_t^-) \in \mathcal{M}_+([0, \pi/2])$ as defined in Section~\ref{sec:1d-red} indeed follows Equation~\eqref{eq:1d-WFR}; then we show that the dynamics are indeed closed by proving that the terms $V_t$ and $G_t$ appearing in the GF depend only on $(\tau_t^+, \tau_t^-)$; and finally, we show that Equation~\eqref{eq:1d-WFR} indeed corresponds to a Wasserstein-Fisher-Rao GF on a given objective functional over $\mathcal{M}_+([0, \pi/2])^2$.

\subsubsection{Proof of the GF equation}\label{app:1d-WFR-equation}
\begin{proof}
We first use the added symmetry to simplify the terms $g_t$ and $\tilde{v}_t$ which appear in the GF with $(\nu_t^+, \nu_t^-)$ (see Section~\ref{sec:low-dim-gen}) and express them only with $||u^H||$ and $||u^\perp||$. Then we use the equations satisfied by $(\nu_t^+, \nu_t^-)$ to obtain equations for $(\tau_t^+, \tau_t^-)$.

% \paragraph{Simplifying $g_t$ and $\tilde{v}_t$.}
% Since, $f^*(x) = ||x^H||$, $f^*$ is now invariant under any orthogonal map $T$ preserving $H$ and $H^\perp$, that is such that the restrictions $T_{|H} \in \mathcal{O}(d_H)$ and $T_{|H^\perp} \in \mathcal{O}(d_\perp)$. Proposition~\ref{th:learning-inv} then ensures that so is $f(\mu_t, \cdot)$. Using Lemma~\ref{th:inv-sub-orthogonal}, it is easy to see that $f(\mu_t; y)$ then reads as $\tilde{f}_t(||x^H||, ||x^\perp||)$ for some $\tilde{f}_t: \mathbb{R}_{+}^2$.

\paragraph{Equations for $(\tau_t^+, \tau_t^-)$.}
Let $\varphi \in \mathcal{C}^1_c([0,\pi/2])$. We have
\begin{align*}
    \partial_t \int \varphi \straightd \tau_t^\pm =& \partial_t \int \varphi \left( \arccos(||u^H||) \right) \straightd \nu_t^\pm(u) \\
    =& \pm \int \nabla_u \left(\varphi \left( \arccos(||u^H||) \right) \right)^\top \tilde{v}_t(u) \straightd \nu_t^\pm(u) \, \\
    &\pm 2 \int \varphi \left( \arccos(||u^H||) \right) g_t(u) \straightd \nu_t^\pm(u).
\end{align*}
One has that 
\begin{align*}
    \nabla_u \left(\varphi \left( \arccos(||u^H||) \right) \right) = \varphi^{\prime} \left( \arccos(||u^H||) \right) \times \frac{-1}{\sqrt{1 - ||u^H||^2}} \frac{u^H}{||u^H||},
\end{align*}
which belongs to $H$. We recall here the expressions of $\tilde{v}_t$ and $g_t$: for any $u \in \mathbb{S}^{d-1}$, we have
\begin{align*}
    g_t(u) &= \int_y R_t(y) \sigma(u^\top y) \straightd \rho(y), \\
    \tilde{v}_t(u) &= \int_y R_t(y) \sigma^\prime(u^\top y)[y - (u^\top y)u] \straightd \rho(y).
\end{align*}
Since, $f^*(x) = ||x^H||$, $f^*$ is now invariant under any orthogonal map $T$ preserving $H$ and $H^\perp$, that is such that the restrictions $T_{|H} \in \mathcal{O}(d_H)$ and $T_{|H^\perp} \in \mathcal{O}(d_\perp)$. Proposition~\ref{th:learning-inv} then ensures that so is $f(\mu_t, \cdot)$, which in turn implies that the residual $R_t(\cdot) = \partial_2 \ell(f(\mu_t; \cdot), f^*(\cdot))$ also shares that invariance property. Using a similar change of variable as in Appendix~\ref{app:low-dim-tau}, and because $\rho$ is spherically symmetric, one gets that $g_t$ can be re-written
\begin{align*}
    g_t(u) = \int_y R_t(y) \sigma \left(y^H_1 ||u^H|| + y^\perp_1 ||u^\perp|| \right) \straightd \rho(y).
\end{align*}
Calling
\begin{align*}
    G_t(\theta) := \int_y R_t(y) \sigma \left(y^H_1 \cos(\theta) + y^\perp_1 \sin(\theta) \right) \straightd \rho(y),
\end{align*}
one has $g_t(u) = G_t(\arccos(||u^H||))$ because $u \in \mathbb{S}^{d-1}$, so that $||u^\perp|| = \sqrt{1 - ||u^H||^2}$. Then, by definition of $\tau_t^\pm$, the second integral in the time derivative above is equal to $\int \varphi(\theta) G_t(\theta) \straightd \tau_t^\pm$. For the first integral appearing in that time derivative, we get
\begin{align*}
    \nabla_u \left(\varphi \left( \arccos(||u^H||) \right) \right)^\top \tilde{v}_t(u) &= \frac{\varphi^\prime \left( \arccos(||u^H||) \right)}{||u^\perp|| \, ||u^H||}  \int_y R_t(y) \sigma'(u^\top y) [(u^\top y) u - y]^\top u^H \straightd \rho.
\end{align*}
Expanding the inner product inside the integral, we have
\begin{align*}
    [(u^\top y) u - y]^\top u^H &= (\langle u^H, y^H \rangle + \langle u^\perp, y^\perp \rangle) ||u^H||^2 - \langle u^H, y^H \rangle \\
    &= ||u^H||^2 \langle u^\perp, y^\perp \rangle - (1 - ||u^H||^2) \langle u^H, y^H \rangle \\
    &= ||u^H||^2 \langle u^\perp, y^\perp \rangle - ||u^\perp||^2 \langle u^H, y^H \rangle.
\end{align*}
Calling
\begin{align*}
    V_t(\theta) := \int_y R_t(y) \sigma^\prime \left(y^H_1 \cos(\theta) + y^\perp_1 \sin(\theta) \right)[y^\perp_1 \cos(\theta) - y^H_1 \sin(\theta)] \straightd \rho(y) = G^\prime(\theta),
\end{align*}
and using again the spherical symmetry of $\rho$, with the same change of variable in the integral as for $g_t$, we get that 
\begin{align*}
    \nabla_u \left(\varphi \left( \arccos(||u^H||) \right) \right)^\top \tilde{v}_t(u) &= \varphi^\prime \left( \arccos(||u^H||) \right) V_t(\arccos(||u^H||)).
\end{align*}
Finally, this combined with the previous result on the integral with $g_t$ yields
\begin{align*}
    \partial_t \int \varphi \straightd \tau_t^\pm =& \pm \int \varphi^\prime(\theta) V_t(\theta) \straightd \tau_t^\pm(\theta) \pm 2 \int \varphi(\theta) G_t(\theta) \straightd \tau_t^\pm(\theta),
\end{align*}
which leads to the desired equation
\begin{align*}
    \partial  \tau_t^\pm = - \text{div} \left(\pm V_t \tau_t^\pm \right) \pm 2 G_t \tau_t^\pm.
\end{align*}

% Using Lemma~\ref{th:inv-sub-orthogonal}, it is easy to see that $f(\mu_t; y)$ then reads as $\tilde{f}_t(||x^H||, ||x^\perp||)$ for some $\tilde{f}_t: \mathbb{R}_{+}^2$.
\end{proof}

\subsubsection[]{Proof that the dynamics on the angle $\theta$ are closed}\label{app:1d-theta}
The proof follow closely that of Appendix~\ref{app:low-dim-tau} (where we prove closed dynamics), except here we take advantage of the added symmetry of the dynamics.
As in Appendix~\ref{app:low-dim-tau}, we have 
\begin{align*}
    f(\mu_t;y) =& \int_{u \in \mathbb{S}^{d-1}} \sigma \left(u^\top y \right) \straightd \left(\nu_t^+ - \nu_t^-\right)(u),
\end{align*}
and the only thing to prove is that this quantity can be expressed using only $(\tau_t^+, \tau_t^-)$. As in Appendix~\ref{app:low-dim-tau}, we first prove this when $\nu_t^\pm$ has a density, which is the case for $\nu_0^\pm$ and should thus remain so during the dynamics. 

Similarly to what occurs in Appendix~\ref{app:low-dim-tau}, $\nu_t^\pm$ is invariant by any orthogonal map $T$ which preserves $H$ and $H^\perp$ because $\mu_t$ has those symmetries given by Proposition~\ref{th:learning-inv}, and if $\nu_t^\pm$ has a density $p_{\nu_t^\pm}$ \wrt~$\tilde{\omega}_d$, then $p^\pm_t$ is also invariant by any such map $T$, and thus depends only on the norms $||u^H||$ and $||u^\perp||$ of its input $u \in \mathbb{S}^{d-1}$. But since its input is on the sphere, those norms are determined by the angle $\theta = \arccos(||u^H||)$ between the input $u$ and $H$. Calling $q_t^\pm$ such that $p_t^\pm(u) = q_t^\pm(\arccos(||u^H||))$, this will lead $\tau_t^\pm$ to have the density $q_t^\pm$ \wrt~$\tilde{\gamma}$. Then, we show below that similarly to Appendix~\ref{app:low-dim-tau}, the integral of any measurable $\varphi:\mathbb{S}^{d-1} \to \mathbb{R}$ against $\nu_t^\pm$ can be expressed as an integral against $\tau_t^\pm$. Indeed, using the disintegration Lemma~\ref{th:disintegration},
\begin{align*}
    \int \varphi d\nu_t^\pm &= \int_{\theta \in [0, \pi/2]} \varphi(u) q_t^\pm(\arccos(||u^H||)) \straightd \tilde{\omega}_d(u) \\
    &= \int_{u} \varphi \left(\cos(\theta) z^H + \sin(\theta) z^\perp \right) q_t^\pm(\theta) \straightd \tilde{\omega}_{d_H}(z^H) \straightd \tilde{\omega}_{d_\perp}(z^\perp) \straightd \tilde{\gamma}(\theta) \\
    &= \int_{\theta \in [0, \pi/2]} \tilde{\varphi}(\theta) q_t^\pm(\theta) \straightd \tilde{\gamma}(\theta) \\
    &= \int_{\theta \in [0, \pi/2]} \tilde{\varphi}(\theta) \straightd \tau_t^\pm(\theta)
\end{align*}
where
\begin{align*}
    \tilde{\varphi}(\theta) := \int_{z^H, z^\perp} \varphi \left(\cos(\theta) z^H + \sin(\theta) z^\perp \right) \straightd \tilde{\omega}_{d_H}(z^H) \straightd \tilde{\omega}_{d_\perp}(z^\perp),
\end{align*}
which concludes the proof if $\nu_t^\pm$ has a density \wrt~the uniform measure $\tilde{\omega}_d$ on the sphere $\mathbb{S}^{d-1}$. The general case is obtained by a weak convergence argument (of measures with density) as in the third step of Section~\ref{app:low-dim-tau}.

\subsubsection{Proof of the Wasserstein-Fisher-Rao GF}\label{app:1d-WFR-GF}
\begin{proof}
Recall that $\gamma$ is the measure in Definition~\ref{def:gamma}, and consider the following objective functional over $\mathcal{M}([0, \pi/2])^2$:
\begin{align*}
    A(\tau^+, \tau^-) :&= \int_{\varphi \in [0, \pi/2]}
    \ell \Big(\cos(\varphi), \, \tilde{f}(\tau^+, \tau^-; \varphi) \Big) \straightd \tilde{\gamma}(\varphi),\\
    \tilde{f}(\tau^+, \tau^-; \varphi) :&= \int_{\theta \in [0, \pi/2]} \tilde{\phi} \left(\theta; \varphi \right) \straightd(\tau^+ - \tau^-)(\theta), \\
    \tilde{\phi} \left(\theta; \varphi \right) :&= \int_{r,s \in [-1,1]} \sigma \Big( r \cos(\varphi) \cos(\theta) + s \sin(\varphi) \sin(\theta) \Big) \straightd \tilde{\gamma}_{d_H}(r) \straightd \tilde{\gamma}_{d_\perp}(s)
\end{align*}
where, for any $p \in \mathbb{N}$, $\straightd \gamma_p(r) = (1-r^2)^{(p-3)/2} \straightd r$, and $\tilde{\gamma}_p = \gamma_p / |\gamma_p|$ with the normalizing factor $|\gamma_p| = B\left(1/2, (p-1)/2 \right) = \sqrt{\pi} \Gamma((p-1)/2) / \Gamma(p/2) = |\mathbb{S}^{p-1}| / |\mathbb{S}^{p-2}|$. Note that $\tilde{\gamma}_p$ can be simply expressed as the law of $\epsilon \times \sqrt{X}$ where $\epsilon \sim \mathcal{U}(\{-1, +1\})$ and $X \sim \text{Beta}(1/2, (p-1)/2)$.

Computing the first variation or Fréchet derivative of the functional $A$ \wrt~to its first and second argument yields, for any $\theta \in [0, \pi/2]$,
\begin{align*}
    \frac{\delta A}{\delta \tau^\pm}(\tau^+, \tau^-)[\theta] &= \pm \int_{\varphi} \partial_2 \ell \Big(\cos(\varphi), \, \tilde{f}(\tau^+, \tau^-; \varphi) \Big) \tilde{\phi}(\theta; \varphi) \straightd \tilde{\gamma}(\varphi).
\end{align*}
To conclude one needs only observe that the quantity above is simply equal to $G_t(\theta)$, up to a fixed multiplicative constant. Since we have assumed $\rho$ to be the uniform measure over $\mathbb{S}^{d-1}$ to ensure that the Wasserstein GF~\eqref{eq:w-gf} is well-defined, the constant is one here but in the case of a general $\rho$ with spherical symmetry, the result should also hold (as long as the Wasserstein GF~\eqref{eq:w-gf} is well-defined) but the proof is more technical and different constants might appear.

% \paragraph{Integrating positively 2-homogeneous functions against $\rho$.}
% First, we show how to simplify the integral of a generic positively 2-homogeneous function $\Phi$ against $\rho$. Note this calculation applies to both $F(\mu_t)$ and $G_t$ by the positive 1-homogeneity of $f^*$, $f(\mu_t; \cdot)$ and $\sigma$. Using the change of variable $y = r u$ with $(r, u) \in \mathbb{R}_+ \times \mathbb{S}^{d-1}$, we have:
% \begin{align*}
%     \int_{y} \Phi(y) \straightd \rho(y) &= \int_{0}^\infty r^{d-1} \int_{u \in \mathbb{S}^{d-1}} r^2 \Phi(u) \frac{e^{-r^2/2}}{(2\pi)^{d/2}} \straightd \omega_d(u) \straightd r.
% \end{align*}
% Separating the integral in $r$ and the integral in $u$, and using the disintegration Lemma~\ref{th:disintegration}, we get:
% \begin{align*}
%     \int_{y} \Phi(y) \straightd \rho(y) &=  \int_{0}^{+\infty} r^{d+1} \frac{e^{-r^2/2}}{(2\pi)^{d/2}} \straightd r |\mathbb{S}^{d-1}| \int_{u \in \mathbb{S}^{d-1}} r^2 \Phi(u) \frac{e^{-r^2/2}}{(2\pi)^{d/2}} \straightd \tilde{\omega}_d(u) \\
%     &= C_d \int \Phi \left(\cos(\varphi) z^H + \sin(\varphi) z^\perp \right)
%     \tilde{\omega}_{d_H}(z^H) \straightd \tilde{\omega}_{d_\perp}(z^\perp) \straightd \tilde{\gamma}(\varphi),
% \end{align*}
% \begin{align*}
%     C_d := \frac{|\mathbb{S}^{d-1}|}{(2\pi)^{d/2}}\int_{0}^{+\infty} r^{d+1} e^{-r^2/2} \straightd r = |\mathbb{S}^{d-1}| \frac{2^{d/2} \Gamma \left(\frac{d}{2} + 1 \right)}{(2\pi)^{d/2}} = |\mathbb{S}^{d-1}| \frac{\Gamma \left(\frac{d}{2} + 1 \right)}{\pi^{d/2}} = d.
% \end{align*}

\paragraph{Simplifying $f(\mu_t; \cdot)$.}
Using the results from Appendix~\ref{app:1d-theta}, we have for any $\varphi, z^H, z^\perp \in [0, \pi/2] \times \mathbb{S}^{d_H-1} \times \mathbb{S}^{d_\perp-1}$ (so that $u = \cos(\varphi) z^H + \sin(\varphi) z^\perp \in \mathbb{S}^{d-1}$)
\begin{align*}
    f(\mu_t;& \cos(\varphi) z^H + \sin(\varphi) z^\perp) = \\
    &\int_{\psi} \int_{\xi^H, \xi^\perp} \sigma \left(\cos(\psi) \cos(\varphi) \langle \xi^H, z^H \rangle + \sin(\psi) \sin(\varphi) \langle \xi^\perp, z^\perp \rangle \right) \straightd \tilde{\omega}_{d_H}(\xi^H) \straightd \tilde{\omega}_{d_\perp}(\xi^\perp) \straightd (\tau_t^+ - \tau_t^-)(\psi)
\end{align*}
Now, because of the integration against uniform measures on the unit spheres, and the inner products involved, we can use some spherical harmonics theory to simplify those calculations. Using The Funk-Hecke formula (see Atkinson and Han~\cite[Theorem 2.22]{atkinsonSpherical}, $n=0$, $d=d_H$ or $d =d_\perp$), we get
\begin{align*}
    f(\mu_t;& \cos(\varphi) z^H + \sin(\varphi) z^\perp) = \\
    &\frac{|\mathbb{S}^{d_H-2}| |\mathbb{S}^{d_\perp-2}|}{|\mathbb{S}^{d_H-1}| |\mathbb{S}^{d_\perp-1}|} \int_{\psi} \int_{r, s} \sigma \left(r \cos(\psi) \cos(\varphi) + s \sin(\psi) \sin(\varphi) \right) \straightd \gamma_{d_H}(r) \straightd \gamma_{d_\perp}(s) \straightd (\tau_t^+ - \tau_t^-)(\psi) \\
    &= \frac{1}{|\gamma_{d_H}| |\gamma_{d\perp}|} |\gamma_{d_H}| |\gamma_{d_\perp}| \int_{\psi \in [0, \pi/2]} \tilde{\phi}(\psi; \varphi) \straightd (\tau_t^+ - \tau_t^-)(\psi) \\
    &=  \tilde{f}(\tau_t^+, \tau_t^-; \varphi).
\end{align*}

\paragraph{Simplifying $f^*(\cos(\varphi) z^H + \sin(\varphi) z^\perp)$.}
Because $f^*(y) = ||y^H||$, $f^*(\cos(\varphi) z^H + \sin(\varphi) z^\perp)$ is simply $||\cos(\varphi) z^H|| = \cos(\varphi)$ because $z^H \in \mathbb{S}^{d_H-1}$. 

With the previous expressions for $f(\mu_t; \cdot)$ and $f^*$ we have that for any function $\Phi: \mathbb{R}^2 \to \mathbb{R}$, 
\begin{align*}
    \Phi \Big(f^*(\cos(\varphi) z^H + \sin(\varphi) z^\perp), \, f(\mu_t; \cos(\varphi) z^H + \sin(\varphi) z^\perp) \Big) = \Phi \Big(\cos(\varphi), \, \tilde{f}(\tau_t^+, \tau_t^-; \varphi) \Big).
\end{align*}
Note that this applies both to $\Phi(y, \hat y) = \ell(y, \hat y)$ and $\Phi(y, \hat y) = -\partial_2 \ell(y, \hat y)$.

\paragraph{Proof that $F(\mu_t) = A(\tau_t^+, \tau_t^-)$.}
%The integrand defining $F(\mu_t)$ is positively 2-homogeneous in $y$ because $f(\mu_t; \cdot)$ and $f^* = ||\,-\,|| \circ p_H$ are positively 1-homogeneous. We thus use the result above to write
Using the disintegration Lemma~\ref{th:disintegration} for the uniform measure on the unit sphere $\mathbb{S}^{d-1}$, we have
\begin{align*}
    F(\mu_t) &=  \int_{y} \ell \Big(f^*(y), \,  f(\mu_t; y) \Big) \straightd \rho(y) \\
    &=  \int  \ell \circ (f^*(\cdot), \, f(\mu_t; \cdot)) \left(\cos(\varphi) z^H + \sin(\varphi) z^\perp \right) \straightd  \tilde{\omega}_{d_H}(z^H) \straightd \tilde{\omega}_{d_\perp}(z^\perp) \straightd \tilde{\gamma}(\varphi) \\
    &= \int \ell \Big(\cos(\varphi), \,  \tilde{f}(\tau_t^+, \tau_t^-; \varphi) \Big) \straightd \tilde{\omega}_{d_H}(z^H) \straightd \tilde{\omega}_{d_\perp}(z^\perp) \straightd \tilde{\gamma}(\varphi) \\
    &= \int  \ell \Big(\cos(\varphi), \,  \tilde{f}(\tau_t^+, \tau_t^-; \varphi) \Big) \straightd \tilde{\gamma}(\varphi),
\end{align*}
where we have used in the last equality the fact that the integrand does not depend on $z^H$ or $z^\perp$ and that $\tilde{\omega}_{d_H}$ and $\tilde{\omega}_{d_\perp}$ are probability measures (and thus their total mass is $1$). 

\paragraph{Simplifying $G_t$.}
Using the disintegration Lemma~\ref{th:disintegration}, we have:
\begin{align*}
    G_t(\theta) &= \int_y R_t(y) \sigma \left(y^H_1 \cos(\theta) + y^\perp_1 \sin(\theta) \right) \straightd \rho(y) \\
    &= \int R_t( \cos(\varphi) z^H + \sin(\varphi) z^\perp)  \sigma \left(z^H_1 \cos(\varphi) \cos(\theta) +  z^\perp_1 \sin(\varphi) \sin(\theta) \right) \tilde{\omega}_{d_H}(z^H) \straightd \tilde{\omega}_{d_\perp}(z^\perp) \straightd \tilde{\gamma}(\varphi).
\end{align*}
Similarly to what we did for simplifying $f(\mu_t; \cdot)$, we can simplify the integrals against $\tilde{\omega}_{d_H}$ and $\tilde{\omega}_{d_\perp}$ using spherical harmonics theory to get:
\begin{align*}
    G_t(\theta) &= - \int_{\varphi \in [0, \pi/2]} \partial_2 \ell \Big(\cos(\varphi), \, \tilde{f}(\tau_t^+, \tau_t^-; \varphi) \Big) \tilde{\phi}(\varphi; \theta) \straightd \tilde{\gamma}(\varphi).
\end{align*}
This shows that 
\begin{align*}
    -\frac{\delta A}{\delta \tau^+}(\tau^+_t, \tau^-_t)[\theta] &= G_t(\theta) \\
    \frac{\delta A}{\delta \tau^-}(\tau^+_t, \tau^-_t)[\theta] &= G_t(\theta),
\end{align*}
which proves that Equation~\eqref{eq:1d-WFR} indeed describes the evolution of the Wasserstein-Fisher-Rao for the objective functional $A$ over $\mathcal{M}([0, \pi/2])^2$, given by the pair $(\tau_t^+, \tau_t^-)$. 
\end{proof}

\section{Numerical simulations in one dimension}\label{app:numerical}

\paragraph{Measure discretization.}
Discretizing $\mu_t$ via $\mu_{m, t} = \frac{1}{m} \sum_{j=1}^m \delta_{(a_j(t), b_j(t))}$, we get that $\tau_{m,t} := \tau_{m, t}^+ - \tau_{m, t}^- = \frac{1}{m} \sum_{j=1}^m c_j(t) \delta_{\theta(t)}$ where 
\begin{align*}
    c_j(t) &= \varepsilon_j |a_j(t)| \, ||b_j(t)||, \\
    \varepsilon_j &= \text{sign}(a_j(0)), \\
    \theta_j(t) &= \arccos \left(\frac{b_j(t)}{||b_j(t)||} \right).
\end{align*}
Initializing through $a_j(0) \sim \mathcal{U}\{-1, +1 \}$ and $b_j(0) \sim \tilde{\omega}_d = \mathcal{U}(\mathbb{S}^{d-1})$, yields $c_j(0) \sim \mathcal{U}\{-1, +1 \}$ and $\theta_j(0) \sim \tilde{\gamma}$, i.i.d.~over $j$. The gradient flows of Equation~\eqref{eq:1d-WFR} translates into the following ODEs on $(c_j)_{j \in [1, m]}$ and $(\theta_j)_{j \in [1, m]}$:
\begin{align*}
    \frac{d}{dt} c_j(t) &= 2 \varepsilon_j G_t(\theta_j(t)) c_j(t), \\
    \frac{d}{dt} \theta_j(t) &=  \varepsilon_j V_t(\theta_j(t)).
\end{align*}
where $\varepsilon_j = a_j(0) \in \{-1, +1\}$ denotes whether the corresponding quantity appears in $\tau_{t,m}^+$ ($\varepsilon=+1$) or $\tau_{t,m}^-$ ($\varepsilon=-1$).

\paragraph{Time discretization.}
Simulating these ODEs via the discrete Euler scheme with step $\eta > 0$, leads, for any iteration $k \in \mathbb{N}$, to:
\begin{equation}\label{eq:num-updates}
    \begin{aligned}
             c_j(k+1) &= \Big(1 + 2 \eta \varepsilon_j G_k(\theta_j(k)) \Big) c_j(k) \\
    \theta_j(k+1) &= \theta_j(k+1) + \eta \varepsilon_j V_k(\theta_j(k)).
    \end{aligned}
\end{equation}

\paragraph{Approximating integrals numerically.}
The only thing that needs to be dealt with numerically is estimating the values of $G_t$ and $V_t$ which are defined by integrals. With the discretization of the measures, we have:
\begin{align*}
    G_k(\theta) &=  \int_{\varphi} \left( \cos(\varphi) - \tilde{f}(\tau_k^+, \tau_k^-; \varphi) \right) \tilde{\phi}(\varphi; \theta) \straightd \tilde{\gamma}(\varphi), \\
    \tilde{f}(\tau_k^+, \tau_k^-; \varphi) &= \sum_{j=1}^m c_j(k) \tilde{\phi}(\theta_j(k); \varphi), \\
    \tilde{\phi} \left(\theta; \varphi \right) &= \int_{r,s \in [-1,1]} \sigma \Big( r \cos(\varphi) \cos(\theta) + s \sin(\varphi) \sin(\theta) \Big) \straightd \tilde{\gamma}_{d_H}(r) \straightd \tilde{\gamma}_{d_\perp}(s).
\end{align*}
We thus get:
\begin{align*}
    G_k(\theta) &= \int \frac{\psi(r,s; \theta, \varphi)}{m} \sum_{j=1}^m \Big(\cos(\varphi) - c_j(k)  \psi(r', s'; \theta_j(k), \varphi) \Big) \straightd\tilde{\gamma}(\varphi) (\straightd\tilde{\gamma}_{d_H})^2(r,r'))(\straightd\tilde{\gamma}_{d_\perp})^2(s,s'),
\end{align*}
with 
\begin{align*}
        \psi(r,s; \theta, \varphi) :&= \sigma \Big( r \cos(\varphi) \cos(\theta) + s \sin(\varphi) \sin(\theta) \Big).
\end{align*}
Similarly, we have:
\begin{align*}
    V_k(\theta) &= \int \frac{\chi(r,s; \theta, \varphi)}{m} \sum_{j=1}^m \Big(\cos(\varphi) - c_j(k)  \psi(r', s'; \theta_j(k), \varphi) \Big) \straightd\tilde{\gamma}(\varphi) (\straightd\tilde{\gamma}_{d_H})^2(r,r'))(\straightd\tilde{\gamma}_{d_\perp})^2(s,s'),
\end{align*}
with 
\begin{align*}
        \chi(r,s; \theta, \varphi) :&= \frac{\partial}{\partial \theta} \psi(r,s; \theta, \varphi) \\
        &= \sigma' \Big( \cos(\theta) \cos(\varphi)  r +  \sin(\theta) \sin(\varphi) s\Big) \Big[-\sin(\theta) \cos(\varphi)  r +  \cos(\theta) \sin(\varphi) s \Big].
\end{align*}
% $\tilde{\phi}$ appears twice in the expression of $G_k$, and in $V_k$ it appears once and its derivative \wrt~$\theta$ (its first argument) once as well. 
We use Monte-Carlo estimation through sampling to approximate the integrals against the five variables $(\varphi, r, r^\prime, s, s^\prime)$ by drawing $N$ samples from the corresponding distributions. We get:
\begin{align*}
    G_k(\theta_j(k)) &\approx \frac{1}{m N} \sum_{i=1}^{N} \sum_{l=1}^m \Psi_{ji} \Big(\cos(\Phi_i) - c_l(k) \tilde{\Psi}_{li} \Big), \\
    \Psi_{ji}(k) &= \psi(R_i, S_i; \theta_j(k), \Phi_i) \\
    \tilde{\Psi}_{ji}(k) &= \psi(R^\prime_i, S^\prime_i; \theta_j(k), \Phi_i),
\end{align*}
and similarly
\begin{align*}
    V_k(\theta_j(k)) &\approx \frac{1}{m N} \sum_{i=1}^{N} \sum_{j=1}^m \chi_{ji} \Big(\cos(\Phi_i) - c_l(k) \tilde{\Psi}_{li} \Big), \\
    \chi_{ji}(k) &= \chi(R_i, S_i; \theta_j(k), \Phi_i), 
\end{align*}
where we have drawn the samples i.i.d.~over $i \in [1, N]$:
\begin{align*}
    \Phi_i &\sim \tilde{\gamma}, \\
    R_i, R^\prime_i &\sim \tilde{\gamma}_{d_H}, \\
    S_i, S^\prime_i &\sim \tilde{\gamma}_{d_\perp}.
\end{align*}

\paragraph{Iterations in the numerical simulation.}
Defining the vectors $c(k) = (c_j(k))_{j \in [1, m]}$, $\theta(k) = (\theta_j)_{j \in [1, m]}$, and $\varepsilon = (\varepsilon_j)_{j \in [1, m]}$, the update Equations~\eqref{eq:num-updates} can then be written in terms of update rules using the matrices $\Psi(k) = (\Psi_{ji}(k))_{j,i \in [1, m] \times [1, N]}$, $\tilde{\Psi}(k) = (\tilde{\Psi}_{ji}(k))_{j,i \in [1, m] \times [1, N]}$, and finally $\chi = (\chi_{ji}(k))_{j,i \in [1, m] \times [1, N]}$, and the vectors $(\Phi, R, R^\prime, S, S^\prime) = (\Phi_i, R_i, R^\prime_i, S_i, S^\prime_i)_{i \in [1, N]}$, which are re-sampled at each iteration $k \in [0, K]$, where $K \in \mathbb{N}$:
\begin{align*}
    c(k+1) &= (1 + 2 \eta \varepsilon \odot \hat{G}_k) \odot c(k), \\
    \theta(k+1) &= \theta(k) + \eta  \varepsilon \odot  \hat{V}_k,
\end{align*}
where 
\begin{align*}
    \hat{G}_k &= \frac{d}{N_b} \Psi \left(\cos(\Phi) - \frac{1}{m} \tilde{\Psi}^\top c(k) \right), \\
    \hat{V}_k &= \frac{d}{N_b} \chi \left(\cos(\Phi) - \frac{1}{m} \tilde{\Psi}^\top c(k) \right),
\end{align*}
and $\odot$ denotes the Hadamard (element-wise) product of two vectors. One can compute the loss through sampling in a similar way.

\paragraph{Experimental value for $\alpha$ and parameters of the numerical simulation.} For the numerical simulations, we fix the number of atoms of the measure (or equivalently the width of the network) to $m = 1,024$, the learning rate to $\eta = 5.10^{-3}$, the number of samples for the Monte-Carlo scheme to $N=1,000$, and the total number of iterations to $K=20,000$. The experimental value for $\alpha$ (see Section~\ref{sec:1d-red}) is computed through $\alpha_{\text{exp}} = \tau_{m, K}^+([0, \pi/2])$, that is 
\begin{align*}
    \alpha_{\text{exp}} &= \frac{1}{m} \sum_{j \in J^+} c_j(K), \\
    J^+ :&= \{j \in [1, m] \ ; \ \varepsilon_j =1\}.
\end{align*}
As mentioned in the main text, the behaviour of the numerical simulation depends a lot on the step-size $\eta$. Some of the differences between our observations and our intuitive description of the limiting model (infinite-width and continuous time) can come from too big a step-size. We have thus run the numerical simulation with $\eta = 2.10^{-5}$ as well, for $K = 230,000$ steps but the same differences still appear (\eg, $\tau_{m, k}^+([0, \pi/2])$ still grows larger than the theoretically expected limit $\alpha$ after some time, albeit by a smaller margin) and after the critical $t^*$, some negative particles seem to go slightly beyond $\pi/2$, even with a very small step-size, a fact which cannot happen for the limiting model. Consequently, in Figure~\ref{fig:dist-total-var}, the first histogram bin right after $\pi/2$ has been merged with the one before.

\end{document}